\documentclass{article}

\usepackage{microtype}
\usepackage{amsmath}
\usepackage{amsthm}
\usepackage{mathtools}
\usepackage{amsfonts}
\usepackage{graphicx}
\usepackage{subfigure}
\usepackage{color}
\usepackage{booktabs} % for professional tables
\usepackage{wrapfig}
\usepackage{bbm}
\usepackage{hyperref}

\newtheorem{theorem}{Theorem}
\newtheorem{corollary}{Corollary}

\newcommand{\omt}[1]{}

\usepackage[accepted]{icml2019}

\icmltitlerunning{Direct Uncertainty Prediction for Medical Second Opinions}

\begin{document}

\twocolumn[
\icmltitle{Direct Uncertainty Prediction for Medical Second Opinions}



\icmlsetsymbol{equal}{*}

\begin{icmlauthorlist}
\icmlauthor{Maithra Raghu}{equal,co,goo}
\icmlauthor{Katy Blumer}{equal,goo}
\icmlauthor{Rory Sayres}{goo}
\icmlauthor{Ziad Obermeyer}{uc}
\icmlauthor{Robert Kleinberg}{co}
\icmlauthor{Sendhil Mullainathan}{chi}
\icmlauthor{Jon Kleinberg}{co}
\end{icmlauthorlist}

\icmlaffiliation{co}{Department of Computer Science, Cornell University}
\icmlaffiliation{goo}{Google Brain}
\icmlaffiliation{uc}{UC Berkeley School of Public Health}
\icmlaffiliation{chi}{Chicago Booth School of Business}

\icmlcorrespondingauthor{Maithra Raghu}{maithrar@gmail.com}

\icmlkeywords{Machine Learning, ICML}

\vskip 0.3in
]



\printAffiliationsAndNotice{\icmlEqualContribution} % otherwise use the standard text.

\begin{abstract}
The issue of disagreements amongst human experts is a ubiquitous one in both machine learning and medicine. In medicine, this often corresponds to doctor disagreements on a patient diagnosis. In this work, we show that machine learning models can be  trained to give \textit{uncertainty scores} to data instances that might result in high expert disagreements. In particular, they can identify patient cases that would benefit most from a \textit{medical second opinion}. Our central methodological finding is that \textit{Direct Uncertainty Prediction} (DUP), training a model to predict an uncertainty score directly from the raw patient features, works better than \textit{Uncertainty Via Classification}, the two-step process of training a classifier and postprocessing the output distribution to give an uncertainty score. We show this both with a theoretical result, and on extensive evaluations on a large scale medical imaging application. 
\end{abstract}

\section{Introduction}
In both the practice of machine learning and the practice of medicine, a serious challenge is presented by disagreements amongst human labels. Machine learning classification models are typically developed on large datasets consisting of $(x_i, y_i)$ (data instance, label) pairs. These are collected \cite{RussakovskyECCV10,Welinder2010Crowdsourcing} by assigning each raw instance $x_i$ to multiple human evaluators, yielding several labels $y_i^{(1)}, y_i^{(2)},...y_i^{(n_i)}$. Unsurprisingly, these labels often have disagreements amongst them and must be carefully aggregated to give a single target value. 

This label disagreement issue becomes a full-fledged clinical problem in the healthcare domain. Despite the human labellers now being highly trained medical experts (doctors), disagreements (on the diagnosis) persist \cite{van2017extent, abrams1994opthal, ICDRStandards, Gulshan2016Retinal, rajpurkar2017chexnet}. One example is \cite{van2017extent}, where agreement between referral and final diagnoses in a cohort of two hundred and eighty patients is studied. \textit{Exact} agreement is only found in $12\%$ of cases, but more concerningly, $21\%$ of cases have significant disagreements. This latter group also turns out to be the most costly to treat. Other examples are given by \cite{daniel2004toman}, a study of tuberculosis diagnosis, showing that radiologists disagree with colleagues $25\%$ of the time, and with \textit{themselves} $20\%$ of the time, and \cite{elmore2015diagnostic}, studying disagreement on cancer diagnosis from breast biopsies.

These disagreements arise not solely from random noise \cite{Rolnick2017Noise}, but from expert judgment and bias. In particular, some patient cases $x_i$ intrinsically contain features that result in greater expert uncertainty (e.g. Figure \ref{fig-data-disagreement}.) This motivates applying machine learning to \textit{predict} which patients are likely to give rise to the most doctor disagreement. We call this the \textit{medical second opinion} problem. Such a model could be deployed to automatically identify patients that might need a second doctor's opinion.

\begin{figure}
\includegraphics[width=\columnwidth]{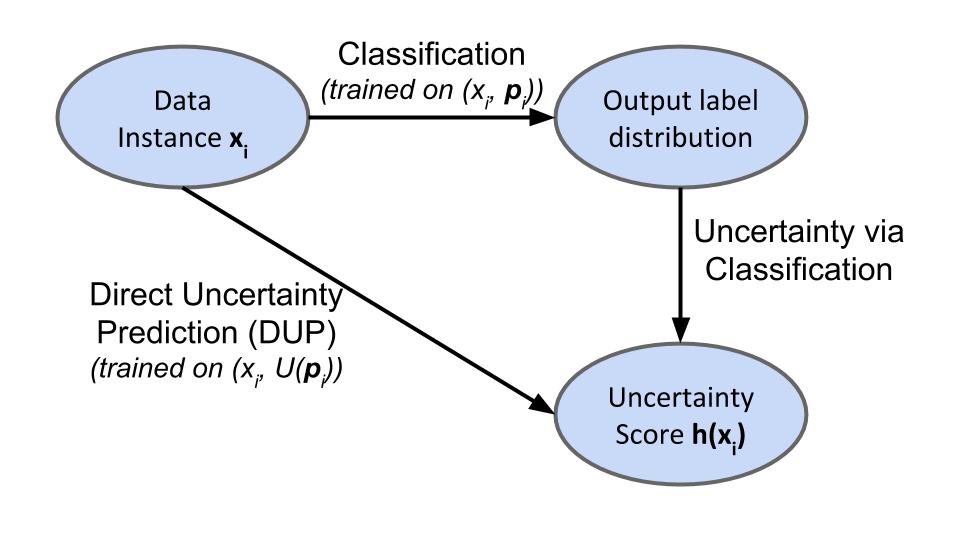}
\vspace{-3.5mm}
\caption{\small \textbf{Different ways of computing an uncertainty scores.} An uncertainty score $h(x_i)$ for $x_i$ can be computed by the two step process of Uncertainty via Classification: training a classifier on pairs (data instance, empirical grade distribution from $y_i^{(j)}$) $(x_i, \mathbf{p}_i)$, and then post processing the classifier output distribution to get an uncertainty score. $h(x_i)$ can also be learned directly on $x_i$, i.e. Direct Uncertainty Prediction. DUP models are trained on pairs (data instance, target uncertainty function on empirical grade distribution), $(x_i, U(\mathbf{p}_i))$. Theoretical and empirical results support the greater effectiveness of Direct Uncertainty Prediction.}
\label{fig:uncertainty-prediction}
\vspace{-3.5mm}
\end{figure}
Mathematically, given a patient instance $x_i$, we are interested in assigning a scalar uncertainty score to $x_i$, $h(x_i)$ that reflects the amount of expert disagreement on $x_i$. For each $x_i$, we have multiple labels $y_i^{(1)}, y_i^{(2)},...y_i^{(n_i)}$, each corresponding to a different individual doctor's grade. 

One natural approach is to first train a classifier mapping $x_i$ to the $y_i^{(j)}$, e.g. via the empirical distribution of labels $\hat{\mathbf{p}}_i$. For ungraded examples, a measure of spread of the output distribution of the classifier (e.g. variance) could be used to give a score. We call this \textit{Uncertainty via Classification (UVC)}. 

An alternate approach, \textit{Direct Uncertainty Prediction (DUP)}, is to learn a function $h_{dup}$ directly mapping $x_i$ to a scalar uncertainty score. The basic contrast with Uncertainty via Classification is illustrated in Figure \ref{fig:uncertainty-prediction}. Our central methodological finding is that Direct Uncertainty Prediction (provably) works better than the two step process of Uncertainty via Classification.

In particular, our three main contributions are the following:
\begin{enumerate}
\item We define simple methods of performing Direct Uncertainty Prediction on data instances $x_i$ with multiple noisy labels. We prove that under a natural model for the data, DUP gives an \textit{unbiased} estimate of the true uncertainty scores $U(x_i)$, while Uncertainty via Classification has a bias term. We then demonstrate this in a synthetic setting of mixtures of Gaussians, and on an image blurring detection task on the standard SVHN and CIFAR-10 benchmarks.
\item We train UVC and DUP models on a large-scale medical imaging task. As predicted by the theory, we find that DUP models perform better at identifying patient cases that will result in large disagreements amongst doctors. 
\item On a small gold standard \textit{adjudicated} test set, we study how well our existing DUP and UVC models can identify patient cases where the individual doctor grade disagrees with a consensus \textit{adjudicated} diagnosis. This adjudicated grade is a proxy for the best possible doctor diagnosis. All DUP models perform better than all UVC models on all evaluations on this task, in both an uncertainty score setting and a ranking application. 
\end{enumerate}
\begin{figure}
  \centering
  \begin{tabular}{cc}
 \includegraphics[width=0.3\columnwidth]{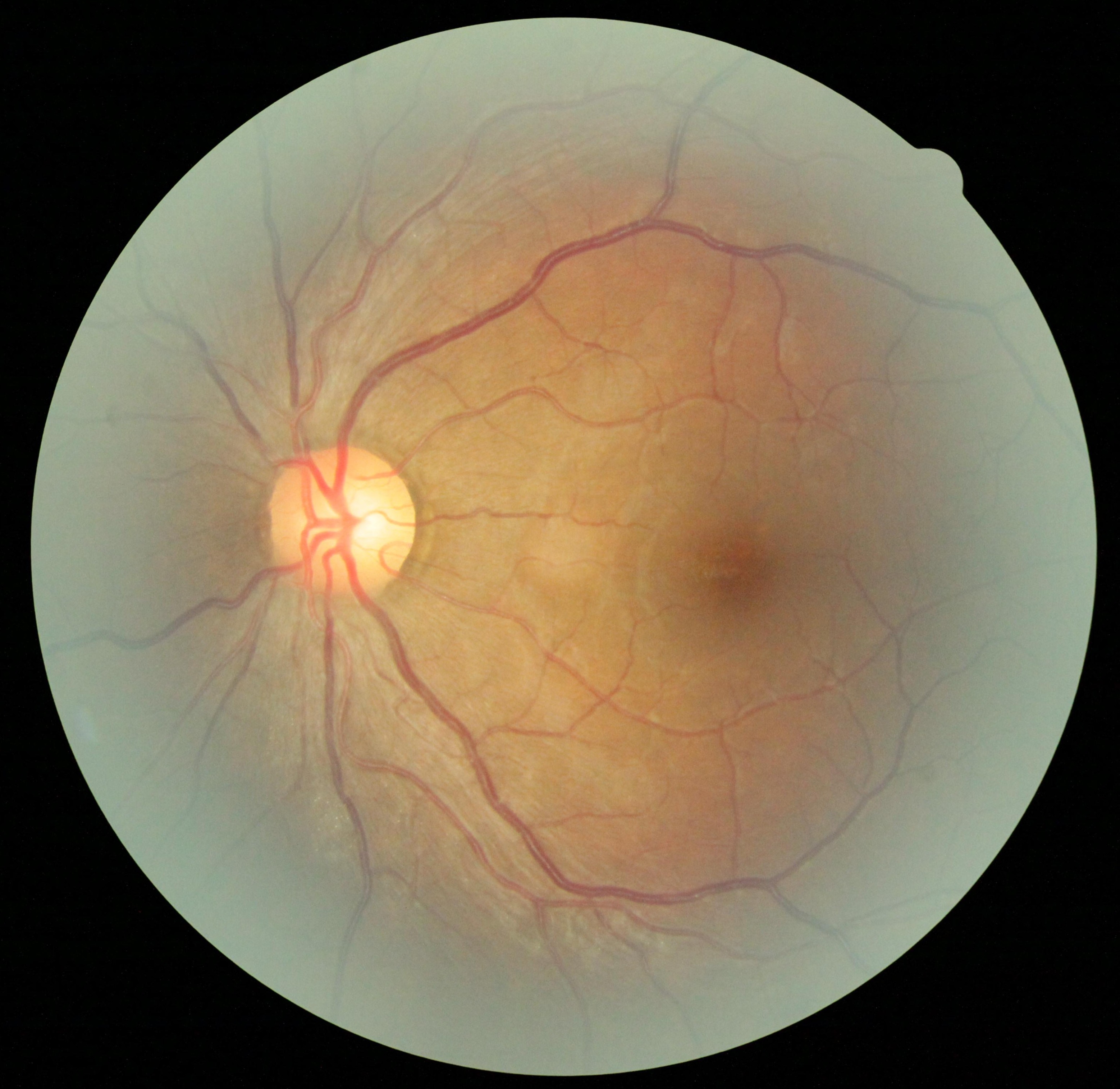}
  &
 \includegraphics[width=0.4\columnwidth]{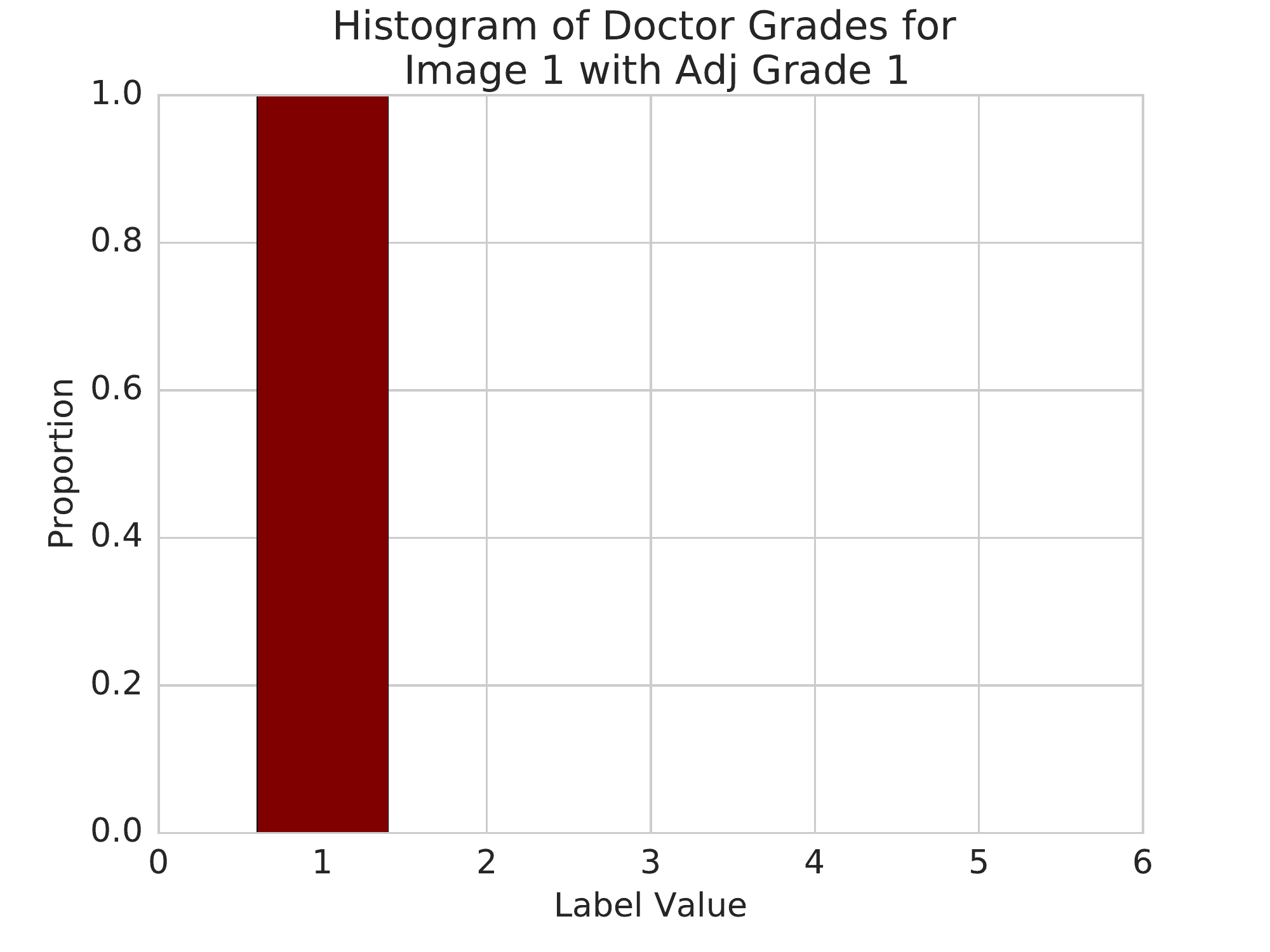}\hspace*{-3mm} \\
 \includegraphics[width=0.3\columnwidth]{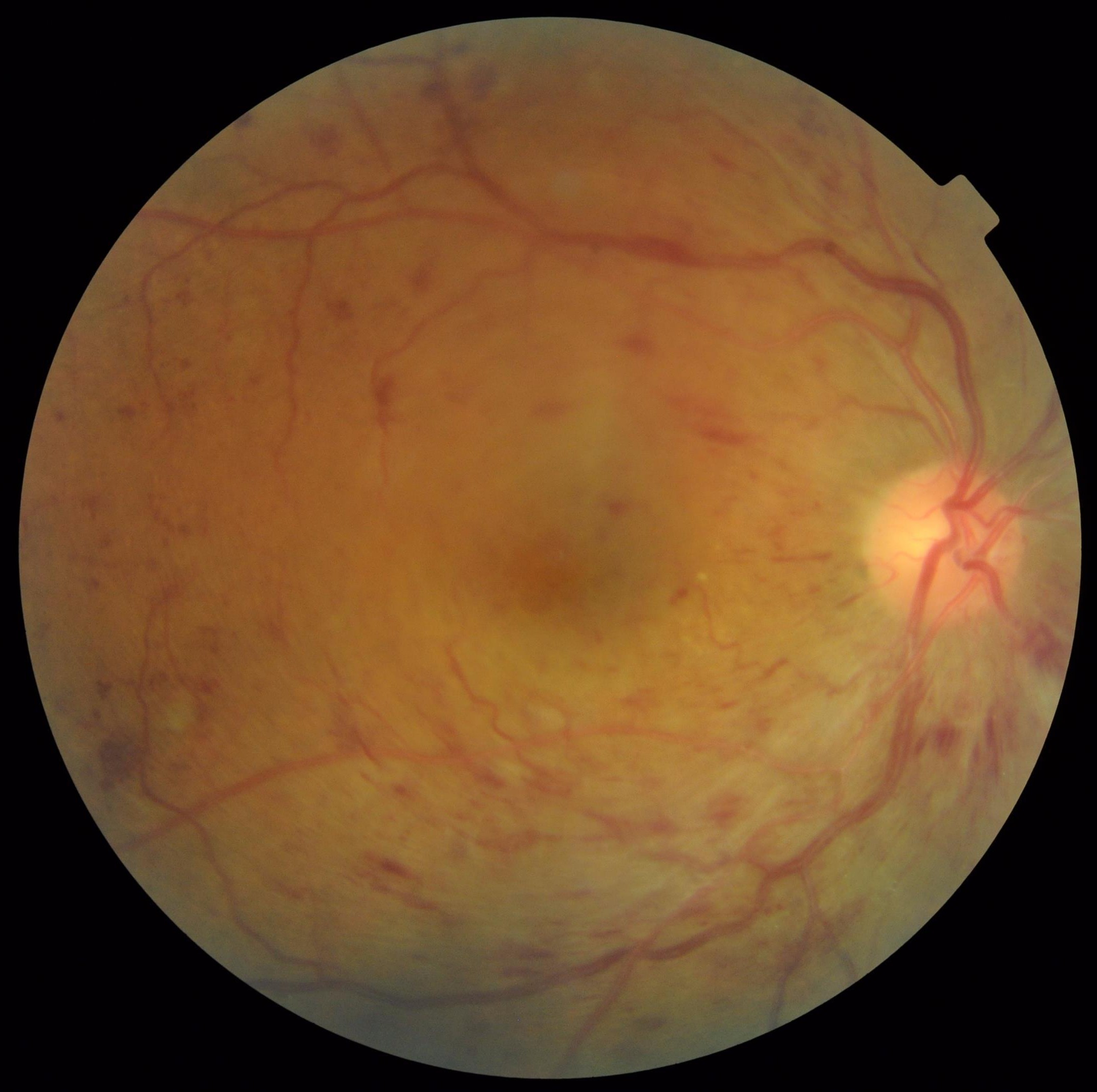}
  &
  \includegraphics[width=0.4\columnwidth]{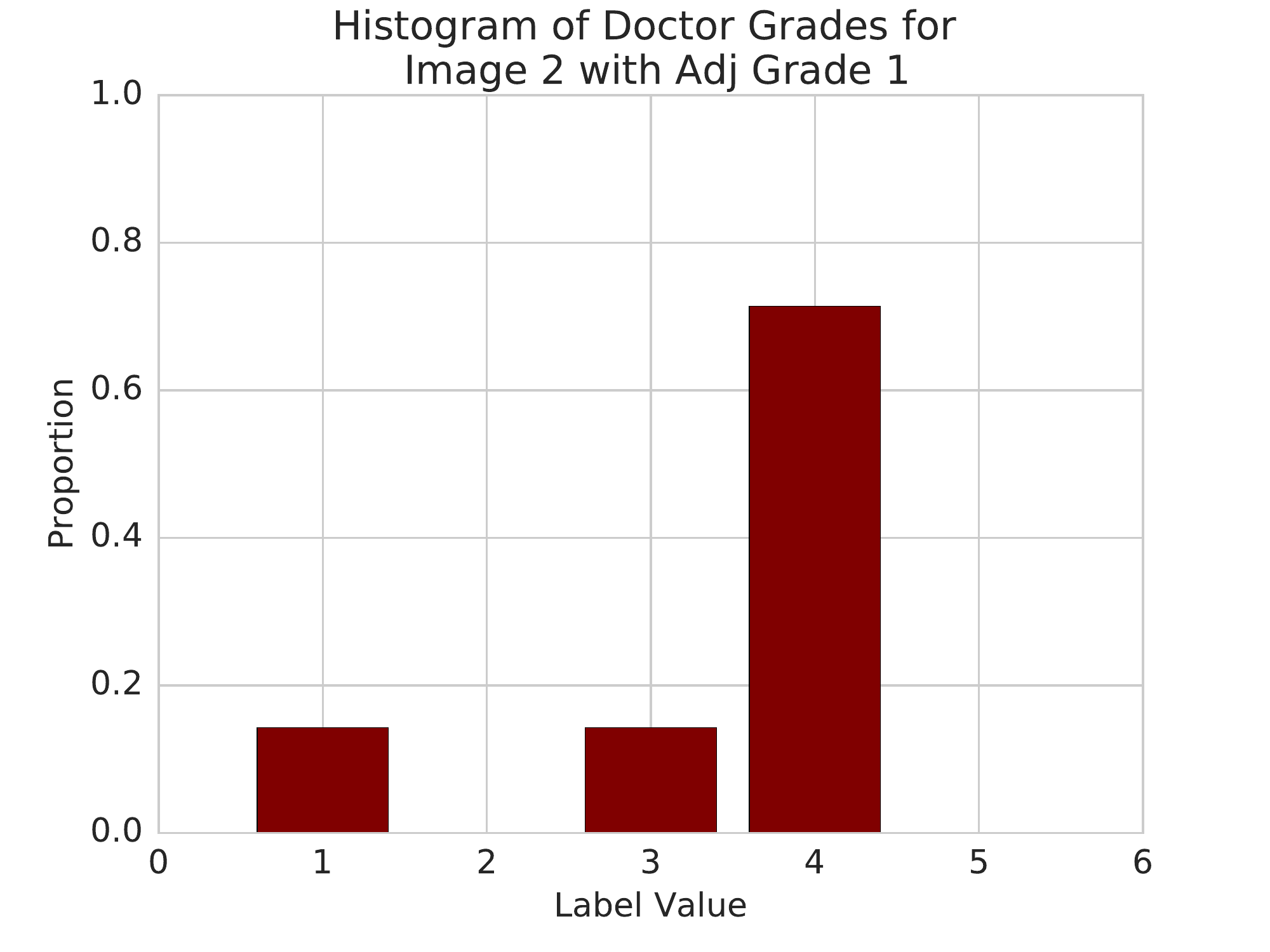}\hspace*{-3mm}
  \end{tabular}
  \caption{\small \textbf{Patient cases have features resulting in higher doctor disagreement.} The two rows give example datapoints in our dataset. The patient images $x_i, x_j$ are in the left column, and on the right we have the empirical probability distribution (histogram) of the multiple individual doctor DR grades. For the top image, all doctors agreed that the grade should be $1$, while there was a significant spread for the bottom image. When later performing an \textit{adjudication} process (Section \ref{sec-adj-evaluation}), where doctors discuss their initial diagnoses with each other and come to a consensus, both patient cases were given an \textit{adjudicated} DR grade of $1$.}
  \label{fig-data-disagreement}
  \vspace{-3.5mm}
\end{figure}
\section{Direct Uncertainty Prediction}
\label{sec-direct-uncertainty-pred}
Our core prediction problem, motivated by identifying patients who need a \textit{medical second opinion}, centers around learning a scalar \textit{uncertainty scoring function} $h(x)$ on patient instances $x$, which signifies the amount of expert disagreement arising from $x$. 

To do so, we must first define a \textit{target} uncertainty scoring function $U(\cdot)$. Our data consists of pairs of the form (patient features, multiple individual doctor labels),  $(x_i; y_i^{(1)}, y_i^{(2)},...y_i^{(n_i)})$ (Figure \ref{fig-data-disagreement}). Letting $c_1,...,c_k$ be the different possible doctor grades, we can define the empirical grade distribution -- the empirical \textit{histogram}: $\mathbf{\hat{p}_i} = [ \hat{p}_i^{(1)},...,\hat{p}_i^{(k)} ]$, with
\[ \hat{p}_i^{(l)} = \frac{\sum_j \mathbbm{1}_{y_i^{(j)} = c_l}}{n_i} \]
Our target uncertainty scoring function $U(\cdot)$ then computes an uncertainty score for $x_i$ using the empirical histogram $\mathbf{\hat{p}_i}$. One such function, which computes the probability that two draws from the empirical histogram will disagree is
\[ U_{disagree}(x_i) = U_{disagree}(\mathbf{\hat{p}_i}) = 1 - \sum_{l=1}^k (\hat{p}_i^{(l)})^2  \tag{1}  \]
Another uncertainty score, which penalizes larger disagreements more, is the variance:
\[ U_{var}(x_i) = U_{var}(\mathbf{\hat{p}_i}) = \sum_{l=1}^k c_l \cdot (\hat{p}_i^{(l)})^2 - \left(\sum c_l \cdot \hat{p}_i^{(l)} \right)^2 (2)  \]

For a large family of these uncertainty scoring functions (including entropy, variance, etc) we can show that Direct Uncertainty Prediction gives an unbiased estimate of $U(\hat{\mathbf{p}}_i)$, whereas uncertainty via classification has a bias term. 

The key observation is that while we want our model to predict doctor disagreement, it does not see all the patient information the doctors do. In particular, the model must predict doctor disagreement based off of only $x_i$ (in our setting, images). In contrast, human doctors see not only $x_i$ but other extensive information, such as patient medical history, family medical history, patient characteristics (age, symptom descriptions, etc) \cite{de2018clinically}. 

Letting $o$ denote all patient features seen by the doctors, we can think of $x_i$ as being the image of $o$ under a (many to one) mapping $g$, which hides the additional patient information, i.e. $x_i = g(o)$. Suppose there are $k$ possible doctor grades, $c_1,..,c_k$. Let $f$ denote the joint distribution over patient features and doctor grades. In particular, let $O$ be a random variable representing patient features, and $Y$ the doctor label for $O$. Then our density function assigns a probability to (patient features, doctor grade) pairs $(o, y)$.

This can also be defined with a vectorized version of the grades: let $Y_l = \mathbf{1}_{Y = c_l}$, the event that $O$ is diagnosed as $c_l$. Then we define the vector $\mathbf{Y} = [Y_1,...,Y_k]$. $f$ is therefore also a density over the points $f(O=o, \mathbf{Y}=\mathbf{y})$. Let the marginal probability of the patient features be $f_O$, with $f_O = \int_{\mathbf{y}} f(O, \mathbf{y})$.

Given an uncertainty scoring function $U(\cdot)$, we would like to predict the disagreement in labels amongst doctors who have seen the patient features $O$. But as the patient features $O$ and doctor grades $\mathbf{Y}$ are jointly distributed according to $f$, this is just the uncertainty of the expected value of $\mathbf{Y}$ under the posterior of $\mathbf{Y}$ given $o$. In particular, we are interested in predicting:
\[ U\left( \int_\mathbf{y} \mathbf{y} \cdot f(\mathbf{Y} = \mathbf{y} | O) \right) = U(\mathbf{E}[\mathbf{Y} | O]) \]
This is a function taking as input a patient's features. For a particular patient's features $o$, we get a scalar uncertainty score given by \[ U(\mathbf{E}[\mathbf{Y} | O = o])\]

However, our model doesn't see $o$, but only $x = g(O)$. We make the mild assumptions that $Y$ is conditionally independent of $g(O)$ given $O$, and that $g(\cdot)$ truly hides information, loosely that $O | g(O)=x$ is not a point mass (see Appendix for details.) In this setting, direct uncertainty prediction, $h_{dup}$ computes the expectation of the uncertainty scores of all the possible posteriors, i.e.
\begin{align*}
h_{dup}(x) &= \mathbb{E} \left[ U(\mathbb{E}[\mathbf{Y}|O]) \vert g(O) = x \right] \\
& = \int_o U(\mathbb{E}[\mathbf{Y} | O=o])f_O(o|g(O) = x)
\end{align*}
Uncertainty via classification $h_{uvc}$ does this in reverse order, first computing the expected posterior, and assigning an uncertainty score to that:
\begin{align*}
h_{uvc}(x) &= U(\mathbb{E} [\mathbf{Y} \vert g(O) = x]) \\
&= U \left( \int_o \mathbb{E}[\mathbf{Y} | O=o]f_O(o|g(O) = x) \right)
\end{align*}
In this setting we can show
\begin{theorem}
\label{thm-dup-unbiased}
Using the above notation
\begin{enumerate}
\item[(i)] $h_{dup}$ is an unbiased estimate of the true uncertainty
\item[(ii)] For any concave uncertainty scoring function $U(\cdot)$ (which includes $U_{disagree}, U_{var}$), uncertainty via classification, $h_{uvc}$ has a bias term.
\end{enumerate}
\end{theorem}
The full proof is the Appendix. A sketch is as follows: the unbiased result arises from the tower rule (law of total expectations). The bias of $h_{uvc}$ follows by the concavity of $U()$, Jensen's inequality, and the fact that $g(\cdot)$ truly hides some patient features. For $U_{disagree}$ and $U_{var}$, we can compute this bias term exactly (full computation in Appendix): 
\begin{corollary}
\label{corollary-dup}
For $U_{disagree}, U_{var}$ the bias term is:
\begin{enumerate}
\item[(i)] Bias of $h_{uvc}$ with $U_{disagree}$: \[ \mathbb{E}_{g(O)} \left[ \sum_l Var_{O|g(O)}\Big(\mathbb{E}[Y_l | O] \Big| g(O) \Big) \right] \] 
\item[(ii)] Bias of $h_{uvc}$ with $U_{var}$: \[ \mathbb{E}_{g(O)} \left[ Var_{O|g(O)}\left(\left. \sum_l l \cdot \mathbb{E}[Y_l | O] \right\vert g(O) \right) \right] \] 
\end{enumerate}
\end{corollary}
In Sections \ref{sec-medical-setting}, \ref{sec-adj-evaluation} we train both Direct Uncertainty Prediction (DUP) models and Uncertainty Via Classification (UVC) models on a large scale medical imaging task. However, to gain intuition for the theoretical results, we first study a toy case on a mixture of Gaussians.

\subsection{Toy Example on Mixture of Gaussians}
To illustrate the formalism in a simplified setting, we consider the following pedagogical toy example.
Suppose our data is generated by a mixture of $k$ Gaussians. Let $f_i \sim \mathcal{N}(\mu_i, \sigma^2_i)$, and $q_i$ be mixture probabilities. Then $f(o, y=i) = q_if_i(o)$ and the marginal $f_O(o) = \sum_{i=1}^kq_if_i(o)$. Additionally, the probability of a particular class $l$ given $o$, $f(y=l|o)$ is simply $\frac{q_lf_l(o)}{\sum_{i=1}^k q_if_i(o)}$.
\begin{table}
  \centering
  \hskip-1.0cm\begin{tabular}{llll}
    \toprule %\hspace*{-5mm}
    % \multicolumn{2}{c}{Part}                   \\
    % \cmidrule(r){1-2}
    \textbf{Model Type} & \textbf{$(3d, 5G)$} & \textbf{$(5d,4G)$} & \textbf{$(10d,4G)$} \\
    \midrule
    UVC &  $69.1\%$ &  $62.0\%$ &  $56.0\%$ \\
    DUP &  $\textbf{74.6}\%$ &  $\textbf{71.2}\%$ &  $\textbf{63.4}\%$ \\
    \bottomrule
  \end{tabular}
  \caption{\small \textbf{DUP and UVC trained to predict disagreement on mixtures of Gaussians.} We train DUP and UVC models on different mixtures of Gaussians, with$(nd, mG)$ denoting a mixture of $m$ Gaussians in $m$ dimensions. Results are in percentage AUC over 3 repeats. The means of the Gaussians are drawn iid from a multivariate normal distribution (full setup in Appendix.) We see that the DUP model performs much better than the UVC model at identifying datapoints with high disagreement in the labels.}
  \label{table-gaussian-mixture}
  \vspace{-3mm}
\end{table}

\textbf{Two 1-D Gaussians:} As the first, most simple case, suppose we have two one dimensional Gaussians, the first, $f_1 = \mathcal{N}(-1, 1)$ and the second, $f_2 = \mathcal{N}(1, 1)$. Assume that the mixture probabilities $q_1, q_2$ are equal to $0.5$. Given $o$ drawn from this mixture $q_1f_1 + q_2f_2$, we'd like to estimate $U(f(y|o))$. Suppose the model sees $x = g(o) = |o|$, the absolute value of $o$. Then, DUP can estimate the uncertainty exactly:
\begin{align*}
\mathbb{E} \left[ U(\mathbb{E}[\mathbf{Y}|O]) \vert x = |o| \right] = & 0.5\cdot U(\mathbb{E}[\mathbf{Y}|O =o]) \\
& + 0.5 \cdot U(\mathbb{E}[\mathbf{Y}|O=-o]) \\
&= U(\mathbb{E}[\mathbf{Y}|O =o]) \\
&= U\Big([f(1|o), 1 - f(1|o)]\Big)
\end{align*}
where the third line follows by the symmetry of the two distributions, with 
\[f(1|o) = \frac{0.5f_1(o)}{0.5f_1(o) + 0.5f_2(o)}\] 
On the other hand, the expected posterior over labels in UVC, $\mathbb{E} [\mathbf{Y} \vert x = |o|]$, is just $[0.5, 0.5]$, as by symmetry, given $x = g(o) = |o|$, $o$ is equally likely to come from $f_1$ or $f_2$. So UVC outputs a constant uncertainty score $U([0.5, 0.5])$ for all $x = |o|$, despite the true varying uncertainty scores. 

\textbf{Training DUPs and UVCs on Mixture of Gaussians:} In Table \ref{table-gaussian-mixture} we train DUPs and UVCs on a few different mixture of Gaussian settings. We generate data $o$ from a Gaussian mixture with iid centers, and labels for the data using the posterior over the different centers given $o$. We use these labels to score $o$ on its uncertainty (using $U_{disagree}$). We then train a model on $x = g(o) = |o|$ to predict whether $x$ is low or high uncertainty. (Full details in Appendix.) We see that DUP performs consistently better than UVC.

\subsection{Example on SVHN and CIFAR-10}
\label{sec-svhn-cifar}
Another empirical demonstration is given by training DUP and UVC to predict label agreement in an image blurring setting. For a source image in SVHN or CIFAR-10, we first apply a Gaussian blur, with a variance chosen for that source image. Then, we draw three noisy labels for the source image, where the noise distribution over labels corresponds to the severity of the image blur. For example, for an image that has a Gaussian blur of variance $0$ (i.e. no blurring), the distribution over labels is a point mass on the true label. For an image that has been blurred severely, there is significant mass on incorrect labels. (Exact distributional details are given in the Appendix.) 
\begin{table}
  \centering
  \hskip-1.0cm\begin{tabular}{lll}
    \toprule %\hspace*{-5mm}
    % \multicolumn{2}{c}{Part}                   \\
    % \cmidrule(r){1-2}
    \textbf{Model} & \textbf{SVHN} (disagree)  & \textbf{CIFAR-10} (disagree)  \\
    \midrule
    UVC &  $75.8\%$ &  $79.1\%$ \\
    DUP &  $\textbf{88.0}\%$ &  $\textbf{85.3}\%$ \\
    \bottomrule
  \end{tabular}
  \caption{\small \textbf{DUP and UVC trained to predict label disagreement corresponding to image blurring on SVHN and CIFAR-10.} DUP outperforms UVC on predicting label disagreement on SVHN and CIFAR-10, where the labels are drawn from a noisy distribution that varies depending on how much blurring the source image has been subjected to. Full details in Appendix.}
  \label{table-image-blur}
  \vspace{-6mm}
\end{table}
We train DUP and UVC models on this dataset and evaluate their ability to predict label disagreement. We again find that DUP models outperfom UVC models. This is despite the setting not directly mapping onto the statement of Theorem \ref{thm-dup-unbiased} -- there is no obscuring function $g$. This suggests the benefits of DUP are more general than the precise theoretical setting. We also observe that the DUP and UVC models learn different features (see Appendix.)

\section{Related Work}
The challenges posed by expert disagreement is an important one, and prior work has put forward several approaches to address some of these issues. Under the assumption that the noise distribution is conditionally independent of the data instance given the true label, \cite{natarajan2013noisy, Sukhbaatar2014Noisy, reed2014noisy, Sheng2008label} provide theoretical analysis along with algorithms to denoise the labels as training progresses, or efficiently collect new labels. However, the conditional independence assumption does not hold in our setting (Figure \ref{fig-data-disagreement}.) Other work relaxes this assumption by defining a domain specific generative model for how noise arises \cite{Mnih_learningto, Xiao2015LearningFM, Veit2017LearningFN} with some methods using additional clean data to pretrain models to form a good prior for learning. Related techniques have also been explored in semantic segmentation \cite{gurari2018predicting, kohl2018probabilistic}. Modeling uncertainty in the context of noisy data has also been looked at through Bayesian techniques \cite{Kendall2017Uncertainties, tanno2017bayesianimagequality}, and (for different models) in the context of crowdsourcing by \cite{Werling2015Learning, Wauthier2011Bayesian}. A related line of work \cite{Dawid79maximumlikelihood, Welinder2010Crowdsourcing} has looked at studying the per labeler error rates, which also requires the additional information of labeler ids, an assumption we relax. Most related is \cite{guan2017labelernoise}, where a multiheaded neural network is used to model different labelers. Surprisingly however, the best model is independent of image features, which is the source of signal in our experiments.

\begin{table*}[t]
  \centering
  \hskip-1.0cm\begin{tabular}{llll}
    \toprule %\hspace*{-5mm}
    % \multicolumn{2}{c}{Part}                   \\
    % \cmidrule(r){1-2}
    \textbf{Task}  &  & \textbf{Model Type} & \textbf{Performance (AUC)} \\
    \midrule
     Variance Prediction  \hspace{5mm} & UVC & Histogram-E2E   & \hspace{10mm} $70.6\%$   \\
     Variance Prediction  \hspace{5mm} & UVC &  Histogram-PC & \hspace{10mm} $70.6\%$     \\
     Variance Prediction  \hspace{5mm} & DUP &  Variance-E2E & \hspace{10mm} $72.9\%$      \\
    %  Variance Prediction      & Variance-E2E-2H & \hspace{5mm} $72.7\%$      \\
     Variance Prediction  \hspace{5mm} & DUP &  Variance-P & \hspace{10mm} $74.4\%$   \\
     Variance Prediction  \hspace{5mm} & DUP &  Variance-PR & \hspace{10mm} $74.6\%$ \\
     Variance Prediction  \hspace{5mm} & DUP &  Variance-PRC & \hspace{10mm} $\textbf{74.8\%}$  \\
     
            \midrule
    Disagreement Prediction \hspace{5mm} & UVC & Histogram-E2E  & \hspace{10mm} $73.4\%$ \\
    Disagreement Prediction \hspace{5mm} & UVC & Histogram-PC & \hspace{10mm} $76.6\%$     \\
    Disagreement Prediction \hspace{5mm} & DUP & Disagree-P & \hspace{10mm} $\textbf{78.1\%}$ \\
    Disagreement Prediction \hspace{5mm} & DUP & Disagree-PC & \hspace{10mm} $\textbf{78.1\%}$ \\
            \midrule
            \midrule
    Variance Prediction   \hspace{5mm} & DUP & Disagree-PC & \hspace{10mm} $73.3\%$    \\
    Disagreement Prediction \hspace{5mm}  & DUP &  Variance-PRC & \hspace{10mm} $77.3\%$ \\ 
    \bottomrule
  \end{tabular}
  \vspace{0mm}
  \caption{\textbf{Performance (percentage AUC) averaged over three runs for UVC and DUPs on Variance Prediction and Disagreement Prediction tasks}. The UVC baselines, which first train a classifier on the empirical grade histogram, are denoted Histogram-. DUPs are trained on either $T_{train}^{(disagree)}$ or $T_{train}^{(var)}$, and denoted Disagree-, Variance- respectively. The top two sets of rows shows the performance of the baseline (and a strengthened baseline Histogram-PC using Prelogit embeddings and Calibration) compared to Variance and Disagree DUPs on the (1) Variance Prediction task (evaluation on $T_{test}^{(var)}$) and (2) Disagreement Prediction task (evaluation on $T_{test}^{(disagree)}$). We see that in both of these settings, the DUPs perform better than the baselines. Additionally, the third set of rows shows tests a Variance DUP on the disagreement task, and vice versa for the Disagreement DUP. We see that both of these also perform better than the baselines.
   \label{table-var-prediction}}
   \vspace{-2mm}
\end{table*}

\section{Doctor Disagreements in DR}
\label{sec-medical-setting}
Our main application studies the effectiveness of Direct Uncertainty Predictors (DUPs) and Uncertainty via Classification (UVC) in identifying patient cases with high disagreements amongst doctors in a large-scale medical imaging setting. These patients stand to benefit most from a medical second opinion.

The task contains patient data in the form of retinal fundus photographs \cite{Gulshan2016Retinal}, large ($587$ x $587$) images of the back of the eye. These photographs can be used to diagnose the patient with different kinds of eye diseases. One such eye disease is Diabetic Retinopathy (DR), which, despite being treatable if caught early enough, remains a leading cause of blindness \cite{ahsan2015dr}.

DR is graded on a $5$-class scale: a grade of $1$ corresponds to \textit{no} DR, $2$ to \textit{mild} DR, $3$ to \textit{moderate} DR, $4$ to \textit{severe} and $5$ to \textit{proliferative} DR \cite{ICDRStandards}. There is an important clinical threshold at grade $3$, with grades $3$ and above corresponding to \textit{referable} DR (needing immediate specialist attention), and $1, 2$ being non-referable. Clinically, the most costly mistake is not referring referable patients, which poses a high risk of blindness.

Our main dataset $T$ has many features typical of medical imaging datasets. $T$ has larger but much fewer images than in natural image datasets such as ImageNet. Each image $x_i$ has a few (typically one to three) individual doctor grades $y_i^{(1)},...,y_i^{(n_i)}$. These grades are also very noisy, with more than $20\%$ of the images  having large (referable/non-referable) disagreement amongst the grades. 

\subsection{Task Setup}
In this section we describe the setup for training variants of DUPs and UVCs using a train test split on $T$. We outline the resulting model performances in Table \ref{table-var-prediction}, which measure how successful the models are in identifying cases where doctors most disagree with each other and consequently where a medical second opinion might be most useful. In Section \ref{sec-adj-evaluation}, we perform a different evaluation (disagreement with consensus) of the best performing DUPs and UVCs on a special, gold standard \textit{adjudicated} test set. In both evaluation settings, we find that DUPs noticeably outperform UVCs.  

The DUP and UVC models are trained and evaluated using a train/test split on $T$, $T_{train}, T_{test}$. This split is constructed using the patient ids of the $x_i \in T$, with $20\%$ of patient ids being set aside to form $T_{test}$ and $80\%$ to form $T_{train}$ (of which $10\%$ is sometimes used as a validation set.) Splitting by patient ids is important to ensure that multiple images $x_i, x_j \in T$ corresponding to a single patient are correctly split \cite{Gulshan2016Retinal}. 

We apply $U_{disagree}(\cdot)$ to the $x_i$ in $T_{train}, T_{test}$ with more than one doctor label to form a new train/test split $T_{train}^{(disagree)}, T_{test}^{(disagree)}$. We repeat this with $U_{var}(\cdot)$ to also form a train/test split $T_{train}^{(var)}, T_{test}^{(var)}$. These two datasets capture the two different medical interpretations of DR grades:

\textit{Categorical Grade Interpretation}: The DR grades can be interpreted as categorical classes, as each grade has specific features associated with it. A grade of $2$ always means microaneurysms, while a grade of $5$ can refer to lesions or laser scars \cite{ICDRStandards}. The $T_{train}^{(disagree)}, T_{test}^{(disagree)}$ data measures disagreement in this categorical setting.

\textit{Continuous Grade Interpretation}: While there are specific features associated with each DR grade, patient conditions tend to progress sequentially through the different DR grades. The $T_{train}^{(var)}, T_{test}^{(var)}$ data thus accounts for the magnitude of differences in doctor grades.

Having formed $T_{train}^{(disagree)}, T_{test}^{(disagree)}$ and $T_{train}^{(var)}, T_{test}^{(var)}$, which consist of pairs $(x_i, U_{disagree}(\mathbf{\hat{p}_i}))$ and $(x_i, U_{var}(\mathbf{\hat{p}_i}))$ respectively, we binarize the uncertainty scores $U_{disagree}(\mathbf{\hat{p}_i}), U_{var}(\mathbf{\hat{p}_i})$ into $0$ (low uncertainty) or $1$ (high uncertainty) to form our final prediction targets. We denote these $U_{disagree}^B(\mathbf{\hat{p}_i}), U_{var}^B(\mathbf{\hat{p}_i})$. More details on this can be found in Appendix Section \ref{sec-uncertainty-details}.

% \begin{figure}
%   \centering
%   \begin{tabular}{cc}
%  \includegraphics[width=0.32\columnwidth]{data_statistics/train_variances.pdf}\hspace*{-3mm}
%   &
%   \includegraphics[width=0.32\columnwidth]{data_statistics/adj_variances.pdf}
%   \end{tabular}
%   \caption{\small TO DO Statistics on dataset $T$.}
%   \label{fig-T-statistics}
% \end{figure}

\subsection{Models and First Experimental Results}
\label{sec-first-experiments}
We train both UVCs and DUPs on this data. All models rely on an Inception-v3 base that, following prior work \cite{Gulshan2016Retinal}, is initialized with pretrained weights on ImageNet. The UVC is performed by first training a classifier $h_c$ on $(x_i, \mathbf{\hat{p}_i})$ pairs in $T_{train}$. The output probability distribution of the classifier, $\tilde{\mathbbm{p}}_i = h_c(x_i)$ is then used as input to the uncertainty scoring function $U(\cdot)$, i.e. $h_{uvc}(x_i) = U \circ h_c(x_i)$ In contrast, the DUPs are trained directly on the pairs $(x_i, U_{disagree}^B(\mathbf{\hat{p}_i})), (x_i, U_{var}^B(\mathbf{\hat{p}_i}))$, i.e. $h_{dup}(x_i)$ directly tries to learn the value of $U^B(\mathbf{\hat{p}_i})$

The results of evaluating these models (on $T_{test}^{(disagree)}$ and $T_{var}^{(disagree)}$) are given in Table \ref{table-var-prediction}. The \textit{Variance Prediction} task corresponds to evaluation on $T_{var}^{(disagree)}$, and the \textit{Disagreement Prediction} task to evaluation on $T_{test}^{(disagree)}$. Both tasks correspond to identifying patients where there is high disagreement amongst doctors. As is typical in medical applications due to class imbalances, performance is given via area under the ROC curve (AUC) \cite{Gulshan2016Retinal, rajpurkar2017chexnet}. 

From the first two sets of rows, we see that DUP models perform better than their UVC counterparts on both tasks. The third set of rows shows the effect of using a variance DUP (Variance-PRC) on the disagreement task and a disagree DUP (Disagree-PC) on the variance task. While these don't perform as well as the best DUP models on their respective tasks, they still beat both the baseline and the strengthened baseline. Below we describe some of the different UVC and DUP variants, with more details in Appendix Section \ref{sec-uncertainty-details}.

\textbf{UVC Models} The UVC models are trained on (image, empirical grade histogram) $(x_i, \hat{\mathbf{p}}_i)$ pairs, and denoted \textit{Histogram-} in Table \ref{table-var-prediction}. The simplest UVC is \textit{Histogram-E2E}, the same model used in \cite{Gulshan2016Retinal}. We improved this baseline by instead taking the prelogit embeddings of Histogram-E2E, and training a small neural network (fully connected, two hidden layers width $300$) with temperature scaling (as in \cite{guo2017calibration}) only on $x_i$ with multiple labels. This gave the strengthened baseline \textit{Histogram-PC}.
\begin{figure}
  \centering
 \includegraphics[width=1.0\columnwidth]{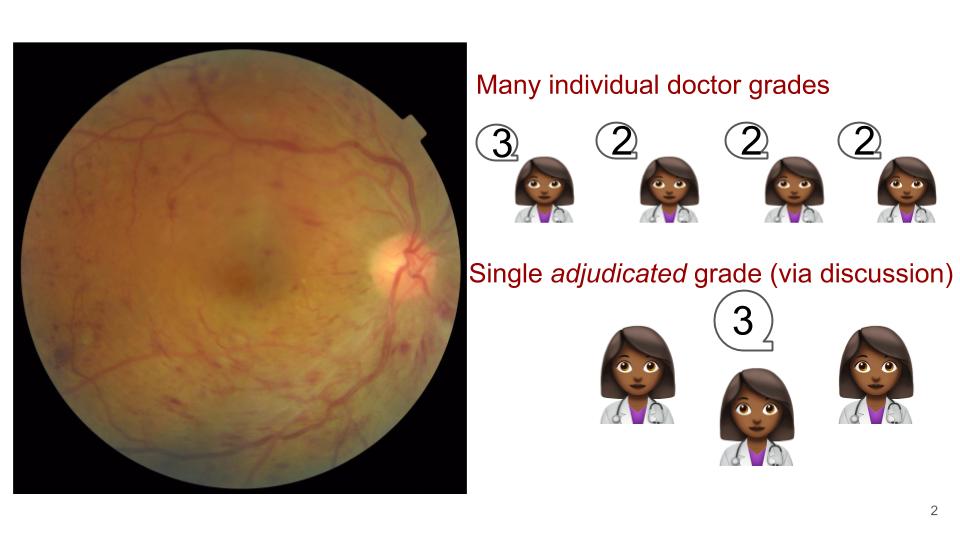}
 \vspace{-3mm}
  \caption{\small \textbf{Labels for the adjudicated dataset $A$.} The small, gold standard adjudicated dataset $A$ has a very different label structure to the main dataset $T$. Each image has many individual doctor grades (typically more than $10$ grades). These doctors also tend to be specialists, with higher rates of agreement. Additionally, each image has a single \textit{adjudicated} grade, where three doctors first grade the image individually, and then come together to discuss the diagnosis and finally give a single, \textit{consensus diagnosis}.}
  \label{fig-adj-process}
  \vspace{-4.5mm}
\end{figure}

\textbf{DUP Variance Models} The simplest Variance DUP is \textit{Variance-E2E}, which is analogous to Histogram-E2E, except trained on $T_{train}^{(var)}$. This performed better than Histogram-E2E, but as $T_{train}^{(var)}$ is small for an Inception-v3, we trained a small neural network (fully connected, two hidden layers width $300$) on the prelogit embeddings, called \textit{Variance-P}. Small variants of \textit{Variance-P} (details in Appendix Section \ref{sec-uncertainty-details}) give \textit{Variance-PR}, and \textit{Variance-PRC}.

\textbf{DUP Disagreement Models} Informed by the variance models, the \textit{Disagree-P} model was designed exactly like the \textit{Variance-P} model (a small fully connected network on prelogit embeddings), but trained on $T^{(disagree)}_{train}$. A small variant of this with calibration gave \textit{Disagree-PC}.

In the Appendix, we demonstrate similar results using entropy as the uncertainty function, as well as experiments studying convergence speed and finite sample behaviour of DUP and UVC. We find that the performance gap between DUP and UVC is robust to train data size, and manifests early in training.

\begin{table*}[t]
  \centering
  \hskip-1.0cm\begin{tabular}{lllllll}
    \toprule %\hspace*{-5mm}
    % \multicolumn{2}{c}{Part}                   \\
    % \cmidrule(r){1-2}
    & \textbf{Model Type} & \textbf{Majority} & \textbf{Median} & \textbf{Majority $=1$} & \textbf{Median $=1$} & \textbf{Referable} \\
    \midrule
    UVC &  Histogram-E2E-Var & \hspace{3mm} $78.1\%$ & \hspace{3mm} $78.2\%$ & \hspace{3mm} $81.3\%$ & \hspace{3mm} $78.1\%$  & \hspace{3mm} $85.5\%$  \\
    UVC &  Histogram-E2E-Disagree & \hspace{3mm} $78.5\%$ & \hspace{3mm} $78.5\%$ & \hspace{3mm} $80.5\%$ & \hspace{3mm} $77.0\%$  & \hspace{3mm} $84.2\%$   \\
    UVC &  Histogram-PC-Var & \hspace{3mm} $77.9\%$ & \hspace{3mm} $78.0\%$ & \hspace{3mm} $80.2\%$ & \hspace{3mm} $77.7\%$  & \hspace{3mm} $85.0\%$  \\
    UVC &  Histogram-PC-Disagree & \hspace{3mm} $79.0\%$ & \hspace{3mm} $78.9\%$ & \hspace{3mm} $80.8\%$ & \hspace{3mm} $79.2\%$ & \hspace{3mm} $84.8\%$ \\
    DUP & Variance-PR & \hspace{3mm} $80.0\%$ & \hspace{3mm} $79.9\%$ & \hspace{3mm} $83.1\%$ & \hspace{3mm} $80.5\%$  & \hspace{3mm} $85.9\%$  \\
    DUP & Variance-PRC & \hspace{3mm} $79.8\%$ & \hspace{3mm} $79.7\%$ & \hspace{3mm} $82.7\%$ & \hspace{3mm} $80.2\%$ & \hspace{3mm} $85.9\%$  \\
    DUP & Disagree-P & \hspace{3mm} $\textbf{81.0\%}$ & \hspace{3mm} $80.8\%$ & \hspace{3mm} $\textbf{84.6\%}$ & \hspace{3mm} $\textbf{81.9\%}$ & \hspace{3mm} $\textbf{86.2\%}$   \\
    DUP & Disagree-PC & \hspace{3mm} $80.9\%$ & \hspace{3mm} $\textbf{80.9\%}$ & \hspace{3mm} $84.5\%$ & \hspace{3mm} $81.8\%$ & \hspace{3mm} $\textbf{86.2\%}$  \\
    \bottomrule
  \end{tabular}
  \caption{\small \textbf{Evaluating models (percentage AUC) on predicting disagreement between an average individual doctor grade and the adjudicated grade.} We evaluate our models's performance using multiple different aggregation metrics (majority, median, binarized non-referable/referable median) as well as special cases of interest (no DR according to majority, no DR according to median). We observe that \textbf{all} direct uncertainty models (Variance-, Disagree-) outperform \textit{all} classifier-based models (Histogram-) on \textit{all} tasks.}
  \label{table-adj-agreement}
  \vspace{-3mm}
\end{table*}

\section{Predicting Disagreement with Consensus: Adjudicated Evaluation}
\label{sec-adj-evaluation}
Section \ref{sec-medical-setting} trained DUPs and UVCs on $T_{train}$, and evaluated them on their ability to identify patient cases where individual doctors were most likely to disagree with each other. Here, we take these trained DUPs/UVCs, and perform an \textit{adjudicated} evaluation, to satisfy two additional goals. 

Firstly, and most importantly, the clinical question of interest is not only in identifying patients where individual doctors disagree with each other, but cases where a more thorough diagnosis -- the \textit{best possible} doctor grade -- would disagree significantly with the \textit{individual} doctor grade. Evaluation on a gold-standard \textit{adjudicated} dataset $A$ enables us to test for this: each image $x_i \in A$ not only has many individual doctor grades (by specialists in the disease) but also a single \textit{adjudicated} grade. This grade is determined by a group of doctors seeking to reach a \textit{consensus} on the diagnosis through discussion \cite{krause2018adj}. Figure \ref{fig-adj-process} illustrates the setup.

We can thus evaluate on this question by seeing if high model uncertainty scores correspond to disagreements between the (average) individual doctor grade and the adjudicated grade. More specifically, we compute different aggregations of the individual doctor grades for $x_i \in A$, and give a binary label for whether this aggregation agrees with the adjudicated grade ($0$ for agreement, $1$ for disagreement). We then see if our model uncertainty scores is predictive of the binary label. 

Secondly, our evaluation on $A$ also provides a more accurate reflection of our models's performance, with less confounding noise. The labels in $A$ (both individual and adjudicated) are much cleaner, with greater consistency amongst doctors. As $A$ is used \textit{solely} for evaluation (all evaluated models are trained on $T_{train}$, Section \ref{sec-medical-setting}), this introduces a distribution shift, but the predicted uncertainty scores transfer well.
\begin{table*}[t]
  \centering
  \begin{tabular}{lllll}
    \toprule %\hspace*{-5mm}
    % \multicolumn{2}{c}{Part}                   \\
    % \cmidrule(r){1-2}
    & \textbf{Prediction Type} & \textbf{Absolute Val} & \textbf{2-Wasserstein} & \textbf{Binary Disagree}  \\
    \midrule
    UVC & Histogram-E2E-Var & \hspace{3mm} $0.650$ & \hspace{3mm} $0.644$ & \hspace{3mm} $0.643$  \\
    UVC & Histogram-E2E-Disagree & \hspace{3mm} $0.645$ & \hspace{3mm} $0.633$ & \hspace{3mm} $0.643$   \\
    UVC & Histogram-PC-Var & \hspace{3mm} $0.638$ & \hspace{3mm} $0.639$ & \hspace{3mm} $0.619$   \\
    UVC & Histogram-PC-Disagree & \hspace{3mm} $0.660$ & \hspace{3mm} $0.655$ & \hspace{3mm} $0.649$  \\
    DUP & Variance-PR & \hspace{3mm} $0.671$ & \hspace{3mm} $0.664$ & \hspace{3mm} $0.660$  \\
    DUP & Variance-PRC & \hspace{3mm} $0.665$ & \hspace{3mm} $0.658$ & \hspace{3mm} $0.656$  \\
    DUP & Disagree-P & \hspace{3mm} $\textbf{0.682}$ & \hspace{3mm} $\textbf{0.670}$ & \hspace{3mm} $\textbf{0.676}$ \\
    DUP & Disagree-PC & \hspace{3mm} $0.680$ & \hspace{3mm} $0.669$ & \hspace{3mm} $0.675$  \\
      \midrule
    &  $2$ Doctors & \hspace{3mm} $0.460$ & \hspace{3mm} $0.448$ & \hspace{3mm} $0.455$  \\
    &  $3$ Doctors & \hspace{3mm} $0.585$ & \hspace{3mm} $0.576$ & \hspace{3mm} $0.580$  \\
    &  $4$ Doctors & \hspace{3mm} $0.641$ & \hspace{3mm} $0.634$ & \hspace{3mm} $0.644$  \\
    &  $5$ Doctors & \hspace{3mm} $0.676$ & \hspace{3mm} $0.670$ & \hspace{3mm} $0.675$  \\
    &  $6$ Doctors & \hspace{3mm} $0.728$ & \hspace{3mm} $0.712$ & \hspace{3mm} $0.718$  \\
    \bottomrule
  \end{tabular}
  \vspace{2mm}
  \caption{\small \textbf{Ranking evaluation of models uncertainty scores using Spearman's rank correlation.} In the top set of rows, we compare the ranking induced by the model uncertainty scores to the (ground truth) ranking induced by the Wasserstein distance between the empirical grade histogram and the adjudicated grade. We use three different metrics for evaluating Wasserstein distance: absolute value distance, 2-Wasserstein and Binary agree/disagree (more details in Appendix \ref{sec-app-wasserstein}.) Again, we see that \textit{all} DUPs outperform \textit{all} baselines on \textit{all} metrics. The second set of rows provides another way to interpret these results. We subsample $n$ doctors to create a new subsampled empirical grade histogram, and compare the ranking induced by the Wasserstein distance between this and the adjudicated grade to the ground truth ranking. We can thus say that the average DUP ranking corresponds to having $5$ doctor grades, and the average UVC ranking corresponds to $4$ doctor grades.  
  \label{table-adj-wasserstein}
   }
   \vspace{-4mm}
\end{table*}
The results are shown in Table \ref{table-adj-agreement}. We evaluate on several different aggregations of individual doctor grades. Like \cite{Gulshan2016Retinal}, we compare agreement between the majority vote of the individual grades and the adjudicated grades. To compensate for a bias of individual doctors giving lower DR grades \cite{krause2018adj}, we also look at agreement between the median individual grade and adjudicated grade. Additionally, we look at referable/non-referable DR grade agreement. We binarize both the individual doctor grades and the adjudicated grade into $0/1$ non-referable/referable, and check agreement between the median binarized grades and adjudicated grade. Finally, we also look at the special case where the average doctor grade is $1$ (no DR), and compare agreement with the adjudicated grade.

We evaluate both baseline models (Histogram-E2E, Histogram-PC) as well as the best performing DUPs, (Variance-PR, Variance-PRC, Disagree-P, Disagree-PC.) The additional -Var, -Disagree suffixes on the baseline models indicate which uncertainty function ($U_{var}$ or $U_{disagree}$) was used to postprocess the classifier output distribution $\tilde{\mathbbm{p}}$ to get an uncertainty score. We find that \textit{all} DUPs outperform \textit{all} the baselines on \textit{all} evaluations.

\subsection{Ranking Evaluation}
A frequent practical challenge in healthcare is to \textit{rank} cases in order of hardest (needing most attention) to easiest (needing least attention), \cite{harrell1984regression}. Therefore, we evaluate how well our models can rank cases from greatest disagreement between the adjudicated and individual grades to least disagreement between the adjudicated and individual grades. To do this however, we need a \textit{continuous} ground truth value reflecting this disagreement, instead of the binary $0/1$ agree/disagree used above. One natural way to do this is to compute the \textit{Wasserstein} distance between the empirical histogram (individual grade distribution) and the adjudicated grade.

At a high level, the Wasserstein distance computes the minimal cost required to move a probability distribution $\mathbf{p}^{(1)}$ to a probability distribution $\mathbf{p}^{(2)}$ with respect to a given metric $d(\cdot)$. In our setting, $\mathbf{p}^{(1)}$ is the empirical histogram $\hat{\mathbf{p}}_i$ of $x_i$, and $\mathbf{p}^{(2)}$ is the point mass at the adjudicated grade $a_i$. When one distribution is a point mass, the Wasserstein distance has a simple interpretation:
\begin{theorem}
Let $\mathbf{p}^{(1)}$ and $\mathbf{p}^{(2)}$ be two probability distributions, with $\mathbf{p}^{(2)}$ a point mass with non-zero value $a$. Let $d(\cdot)$ be a given metric. The Wasserstein distance between $\mathbf{p}^{(1)}$ and $\mathbf{p}^{(2)}$, $|| \mathbf{p}^{(1)} - \mathbf{p}^{(2)}||_{w}$ with respect to $d(\cdot)$ can be written as
\[ || \mathbf{p}^{(1)} - \mathbf{p}^{(2)}||_{w} = \mathbb{E}_{C \sim \mathbf{p}^{(1)}} [ d(C, a)] \]
\end{theorem}

The proof is in Appendix \ref{sec-app-wasserstein}. In our setting, the theorem says that the (continuous) disagreement score for $x_i \in A$ is just the expected distance between a grade drawn from the empirical histogram and the adjudicated grade. We consider three different distance functions $d(\cdot)$: (a) the absolute value of the grade difference, (b) the $2-Wasserstein$ distance, a metrization of the squared distance penalizing large grade differences more (details in Appendix \ref{sec-app-wasserstein}) and (c) a $0/1$ binary agree/disagree metric, in line with the categorical and continous interpretations of DR grades, Section \ref{sec-medical-setting}.

We compare the ranking induced by this continuous disagreement score on $A$ with the ranking induced by the model's predicted uncertainty scores. To evaluate the similarity of these rankings, we use Spearman's rank correlation \cite{spearman1904}, which takes a value between $[-1, 1]$. A $-1$ indicates perfect negative rank correlation, $1$ a perfect positive rank correlation and $0$ no correlation. The results are shown in Table \ref{table-adj-wasserstein}. Similar to Table \ref{table-adj-agreement}, we observe strong performance with DUPs: all DUPs beat all the baselines on all the different distances. 

This task also enables a natural comparison between the models and doctors. In particular, we can compute a third ranking over $A$, by sampling $n$ individual doctor grades, and computing the Wasserstein distance between this subsampled empirical histogram and the adjudicated grade. This experiment tells us how many doctor grades are needed to give a ranking as accurate as the models. For DUPs, we need on average $5$ doctors, while for the UVC baseline, we need on average $4$ doctors.

\section{Discussion}
In this paper, we show that machine learning models can successfully be used to predict data instances that give rise to high expert disagreement. The main motivation for this prediction problem is the medical domain, where some patient cases result in significant differences in doctor diagnoses, and may benefit greatly from a medical second opinion. We show, both with a formal result and through extensive experiments, that Direct Uncertainty Prediction, which learns an uncertainty score directly from the raw patient features, performs significantly better than Uncertainty via Classification. Future work might look at transferring these techniques to different data modalities, and extending the applications to machine learning data denoising.

\clearpage
\subsubsection*{Acknowledgments}
We thank Varun Gulshan and Arunachalam Narayanaswamy for detailed advice and discussion on the model, and David Sontag and Olga Russakovsky for specific suggestions. We also thank Quoc Le, Martin Wattenberg, Jonathan Krause, Lily Peng and Dale Webster for general feedback. We thank Naama Hammel and Zahra Rastegar for helpful medical insights. Robert Kleinberg's work was partially supported by NSF grant CCF-1512964.

\bibliographystyle{icml2019}
\bibliography{refs}

\clearpage

\onecolumn
\appendix
\section*{Appendix}

\section{Proofs of Direct Uncertainty Prediction Results}
We first prove Theorem \ref{thm-dup-unbiased}. 

\begin{proof}
To show unbiasedness of $h_{dup}$, we need to show that $\mathbb{E}[h_{dup}] = \mathbb{E}[U(\mathbf{Y})]$. But from the tower law (law of total expectation):
\[ \mathbb{E}[h_{dup}] = \mathbb{E}\Big[\mathbb{E}\big[U(\mathbb{E}[\mathbf{Y}|O]) \big| g(O)\big]\Big] = \mathbb{E}[U(\mathbb{E}[\mathbf{Y}|O])] \]
To prove the biasedness of $h_{uvc}$, first note that
\[ h_{uvc} = U(\mathbb{E}[\mathbf{Y}|g(O)]) = U(\mathbb{E}[\mathbb{E}[\mathbf{Y}|O]|g(O)]) \]
by the conditional independence of $Y, g(O)$ given $O$.
Next, by the fact that $U(\cdot)$ is concave and Jensen's inequality, 
\[ h_{uvc} = U(\mathbb{E}[\mathbb{E}[\mathbf{Y}|O]|g(O)]) \geq \mathbb{E}[U(\mathbb{E}[\mathbf{Y}|O] | g(O)] = h_{dup} \]
This is a strict inequality whenever the distribution of posteriors induced by conditioning on $g(O)$ is not a point-mass. Therefore we have that $h_{uvc}$ overestimates the the true uncertainty $U(\mathbf{Y})$.
\end{proof}

For specific $U(\cdot)$, we can compute the bias term by first computing $h_{uvc} - h_{dup}$ and then taking an expectation. For $U_{disagree}$, we have
\[ h_{uvc} = U(\mathbb{E}[\mathbf{Y}|g(O)]) = 1 - \sum_l \mathbb{E}[\mathbb{E}[Y_l |O] | g(O)]^2 \]
and 
\[ h_{dup} = \mathbb{E}[U(\mathbb{E}[\mathbf{Y}|O]) | g(O)] = 1 - \sum_l \mathbb{E}[\mathbb{E}[Y_l |O]^2 | g(O)] \]
And so, 
\[ h_{uvc} - h_{dup} = \sum_l \mathbb{E}[\mathbb{E}[Y_l |O]^2 | g(O)] - \sum_l \mathbb{E}[\mathbb{E}[Y_l |O] | g(O)]^2 \]
But this is just
\[ Var\left(\sum_l \mathbb{E}[\mathbb{E}[Y_l |O]|g(O)]\right) \]
Taking expectations over values of $g(O)$ gives the bias, i.e.
\[ \mathbb{E}\left[ Var\left(\sum_l \mathbb{E}[\mathbb{E}[Y_l |O]|g(O)]\right) \right] \]

For $U_{var}$, we have
\[ h_{uvc} = \sum_l l^2 \mathbb{E}[\mathbb{E}[Y_l |O] | g(O)] - \left(\sum_l l \mathbb{E}[\mathbb{E}[Y_l |O] | g(O)] \right)^2 \]
and 
\[ h_{dup} = \mathbb{E}\left[\left. \sum_l l^2 \mathbb{E}[Y_l |O] - \left(\sum_l l \mathbb{E}[Y_l |O] \right)^2 \right\vert g(O)\right] \]
And so $h_{uvc} - h_{dup}$ becomes
\[ \mathbb{E}\left[ \left. \left(\sum_l l \mathbb{E}[Y_l |O] \right)^2 \right\vert g(O)\right] - \left(\sum_l l \mathbb{E}[\mathbb{E}[Y_l |O] | g(O)] \right)^2 \]
Which is just
\[ Var\left(\left. \sum_l l\cdot \mathbb{E}[Y_l |O] \right\vert g(O)\right) \]
Taking expectations over values of $g(O)$ like before gives the result.

\section{Mixture of Gaussians Setting}
We train a DUP and UVC (Figure \ref{table-gaussian-mixture}) on a synthetic task where data is generated from a mixture of Gaussians. All of our settings have uniform mixtures of Gaussians, with the Gaussian mean vectors being drawn from $\mathcal{N}(0, 1/d)$ ($d$ corresponding to the dimension) so that in expectation, each mean vector has norm $1$. The variance is set to be the identity. Like before, we set $x = g(o) = |o|$. We draw five labels for each $x$ from the posterior distribution over Gaussian centers given $x$, and apply $U_{disagree}()$ to the empirical histogram. As with the medical imaging application, we threshold these uncertainty scores (with threshold $0.5$) to give a binary low uncertainty/high uncertainty label, which we use to train our DUPs and UVCs. Results are given in percentage AUC to account for some settings having unbalanced classes.

We train fully connected networks with two hidden layers of width $300$ on this task, using the SGD with momentum optimizer and an initial learning rate of $0.01$. 

\section{SVHN and CIFAR-10 Setting}
In Section \ref{sec-svhn-cifar}, we train DUP and UVC models to predict label disagreement on a synthetic task on SVHN and CIFAR-10. The task setup is as follows: for each image in SVHN/CIFAR-10, we decide on a variance ($0, 1, 2, 3$) for a Gaussian filter that is applied to the image. Three labels are then drawn for the image from a noisy distribution over labels, with the label noise distribution depending on the variance of the Gaussian filter. Specifically, for a Gaussian filter with variance $0$, the noise distribution is just a point mass on the true label. For a Gaussian filter with variance $1$, the three labels are drawn from a distribution with $0.02$ mass on four incorrect labels, and the remaining $0.92$ mass on the correct label. For variance equal to $2$, the labels are drawn from a distribution with $0.08$ mass on four different labels, and remaining mass on the true label. For variance $3$, this mass is now $0.12$ on the incorrect labels.

A simple conv network, with $3$x$3$ kernels and channels $64--128--256$, followed by fully connected layers of width $1000$ and $200$ (each with batch normalization) is trained on this dataset, with the UVC model trained on the empirical histogram, and the DUP model trained on a binary agree/disagree target. (Disagreement threshold is if at least one label disagrees.) We find that DUP outperforms UVC on both SVHN and CIFAR-10.

\textbf{Learned Features} Interestingly, we also observe that the features learned by the DUP and UVC models are different to each other. We apply saliency maps, specifically SmoothGrad \cite{smilkov2017smoothgrad} and IntGrad \cite{sundararajan2017axiomatic} to study the features that DUP and UVC pay attention to in the input image. 

\begin{figure}
\centering
\begin{tabular}{ccccc}
\hspace*{-10mm} \includegraphics[width=0.35\columnwidth]{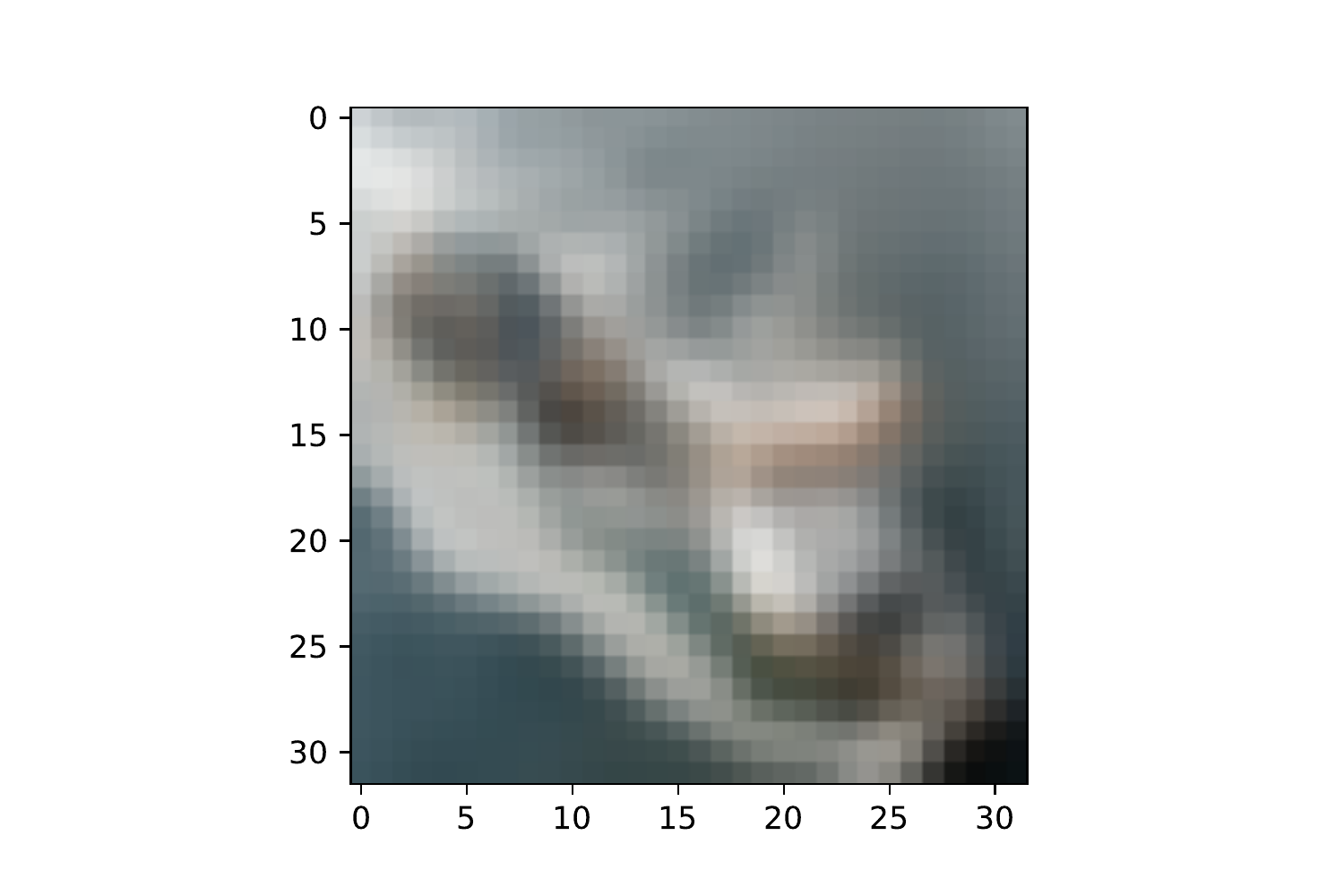} &
\hspace*{-10mm} \includegraphics[width=0.2\columnwidth]{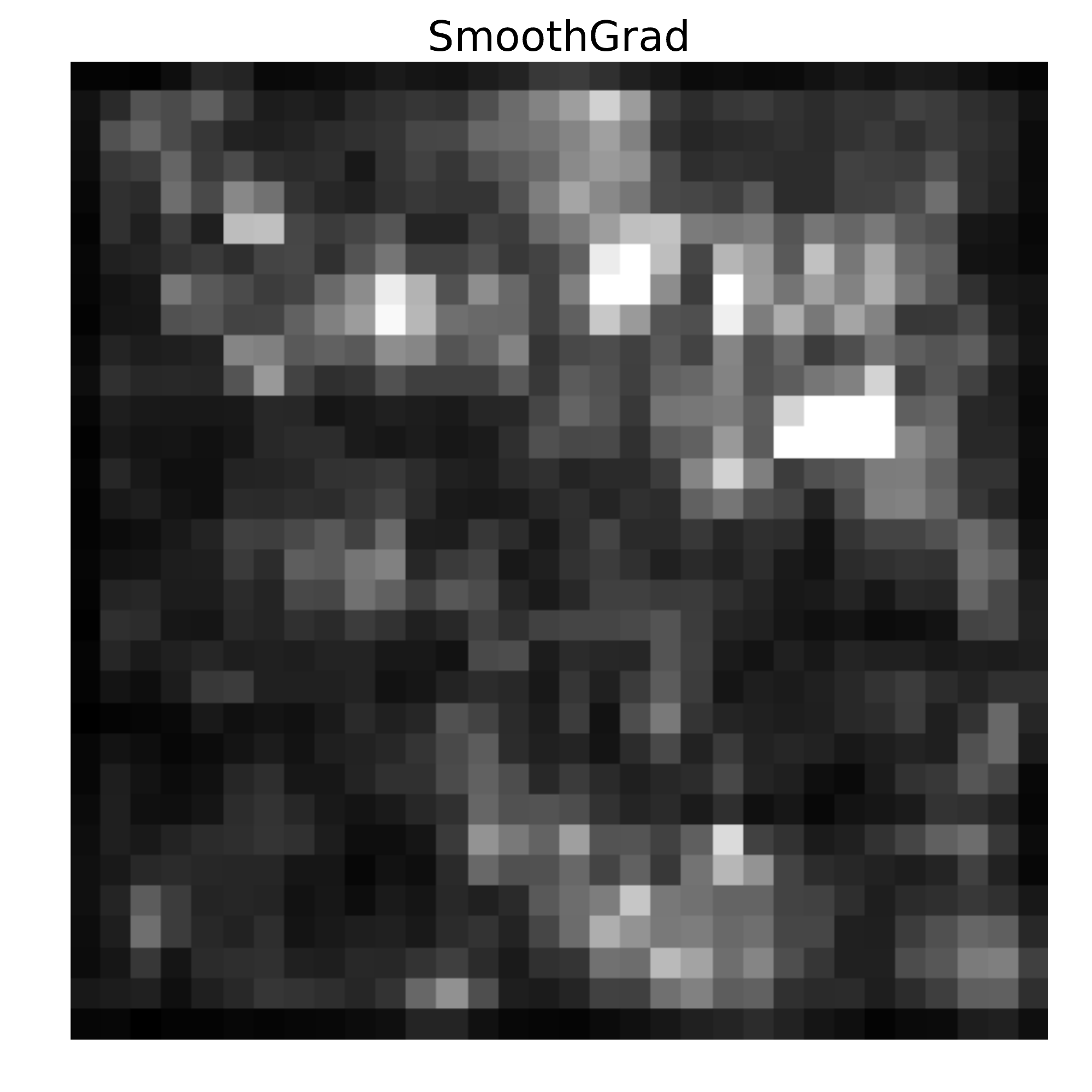} &
\hspace*{-5mm} \includegraphics[width=0.2\columnwidth]{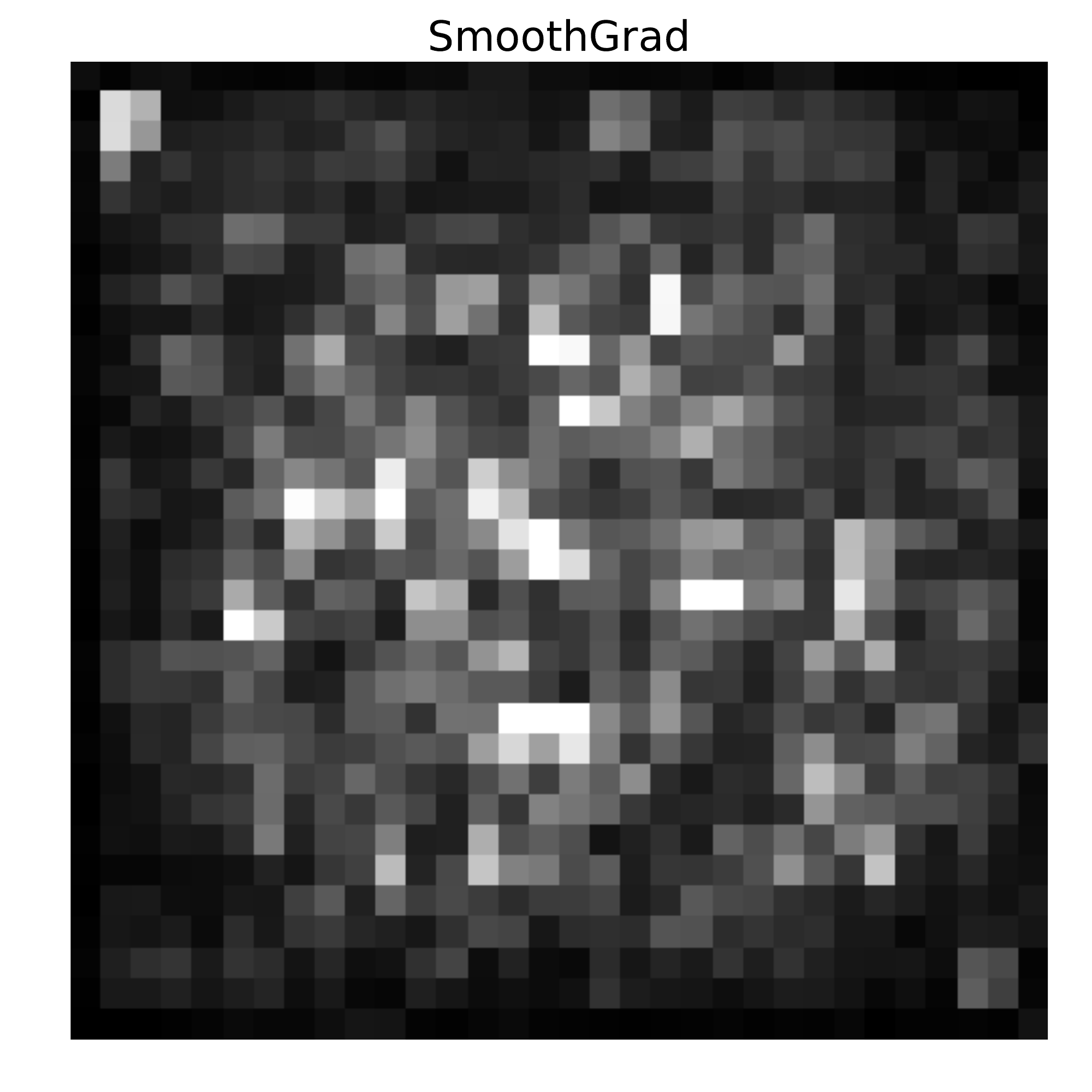} &
\hspace*{-5mm} \includegraphics[width=0.2\columnwidth]{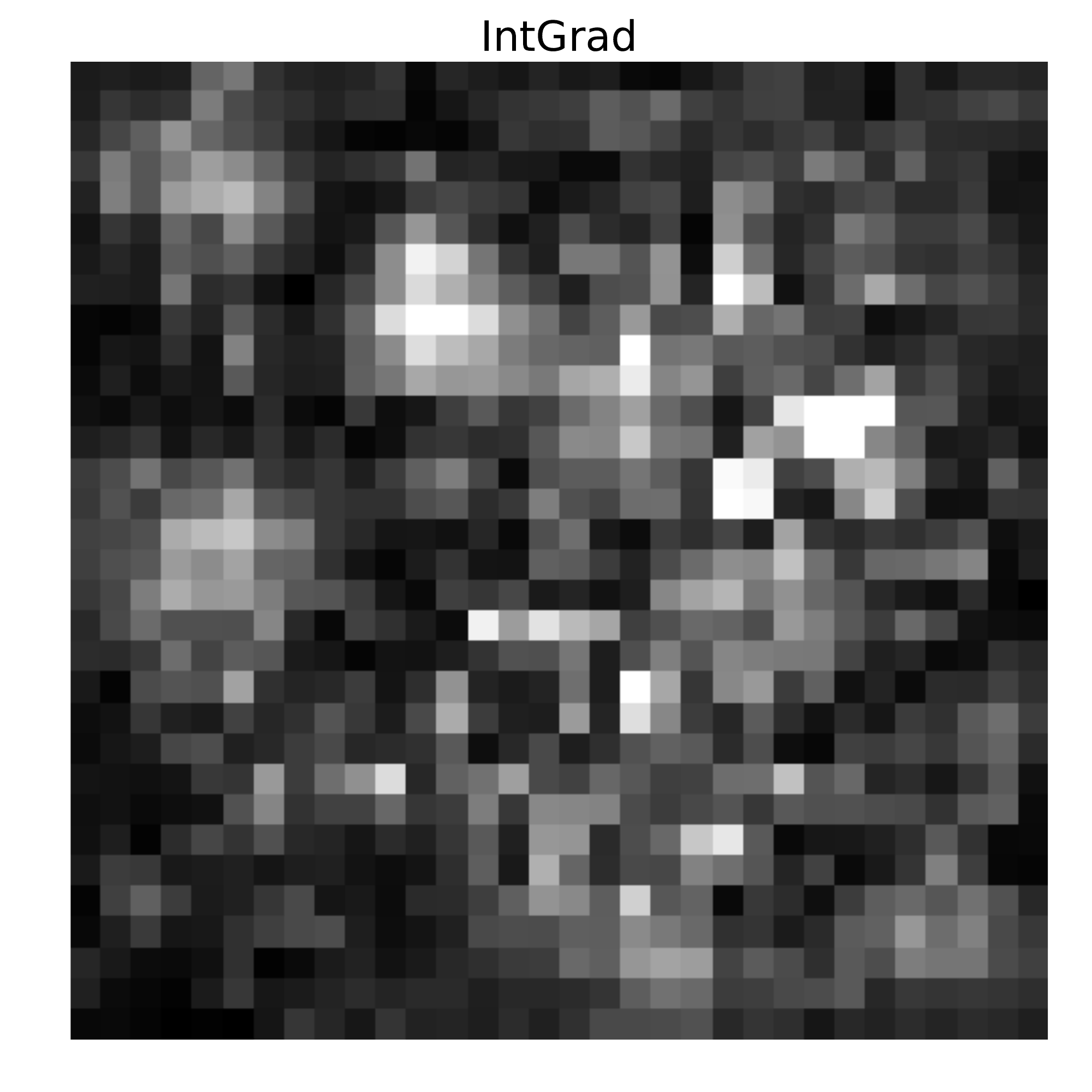} &
\hspace*{-5mm} \includegraphics[width=0.2\columnwidth]{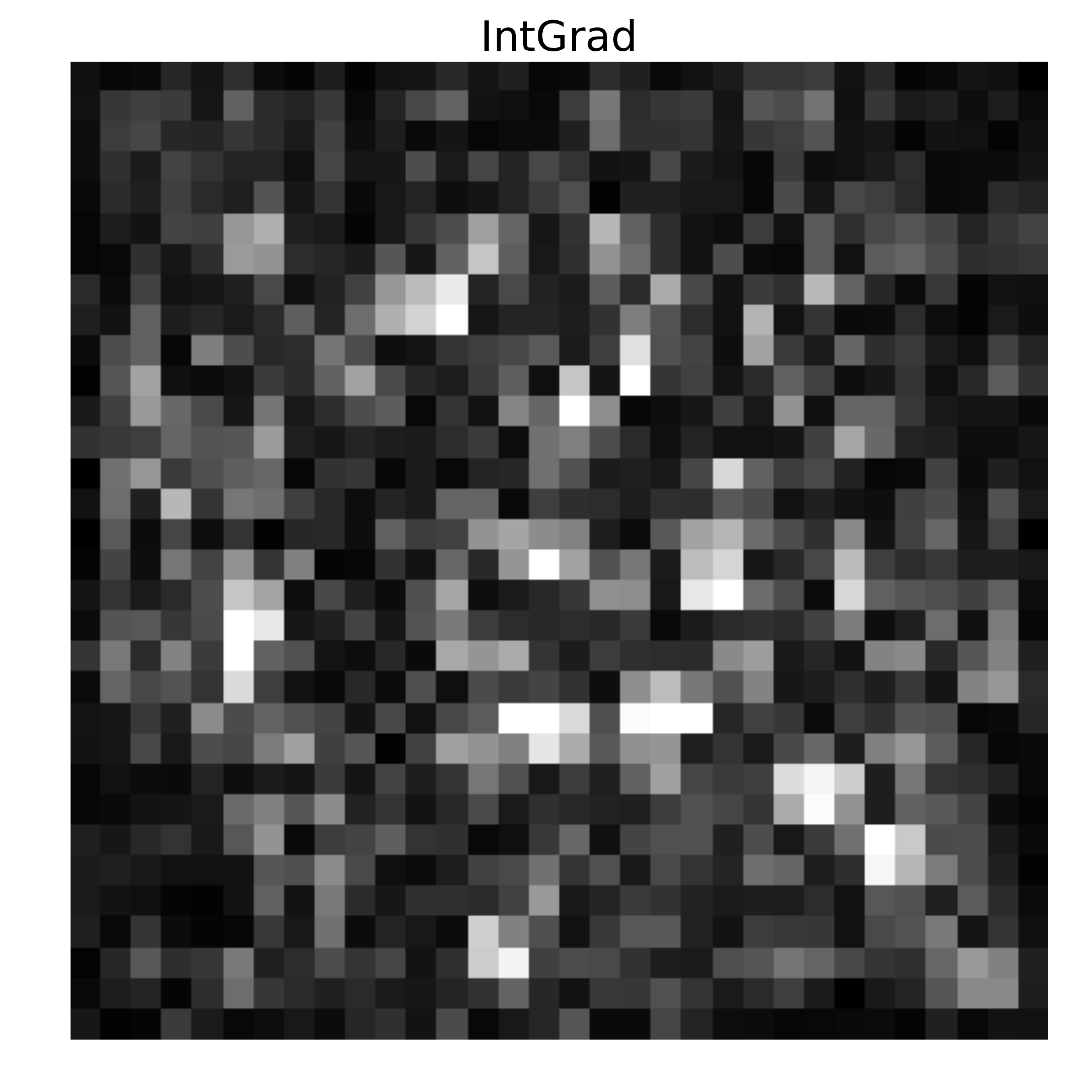} \\

\hspace*{-10mm} \includegraphics[width=0.35\columnwidth]{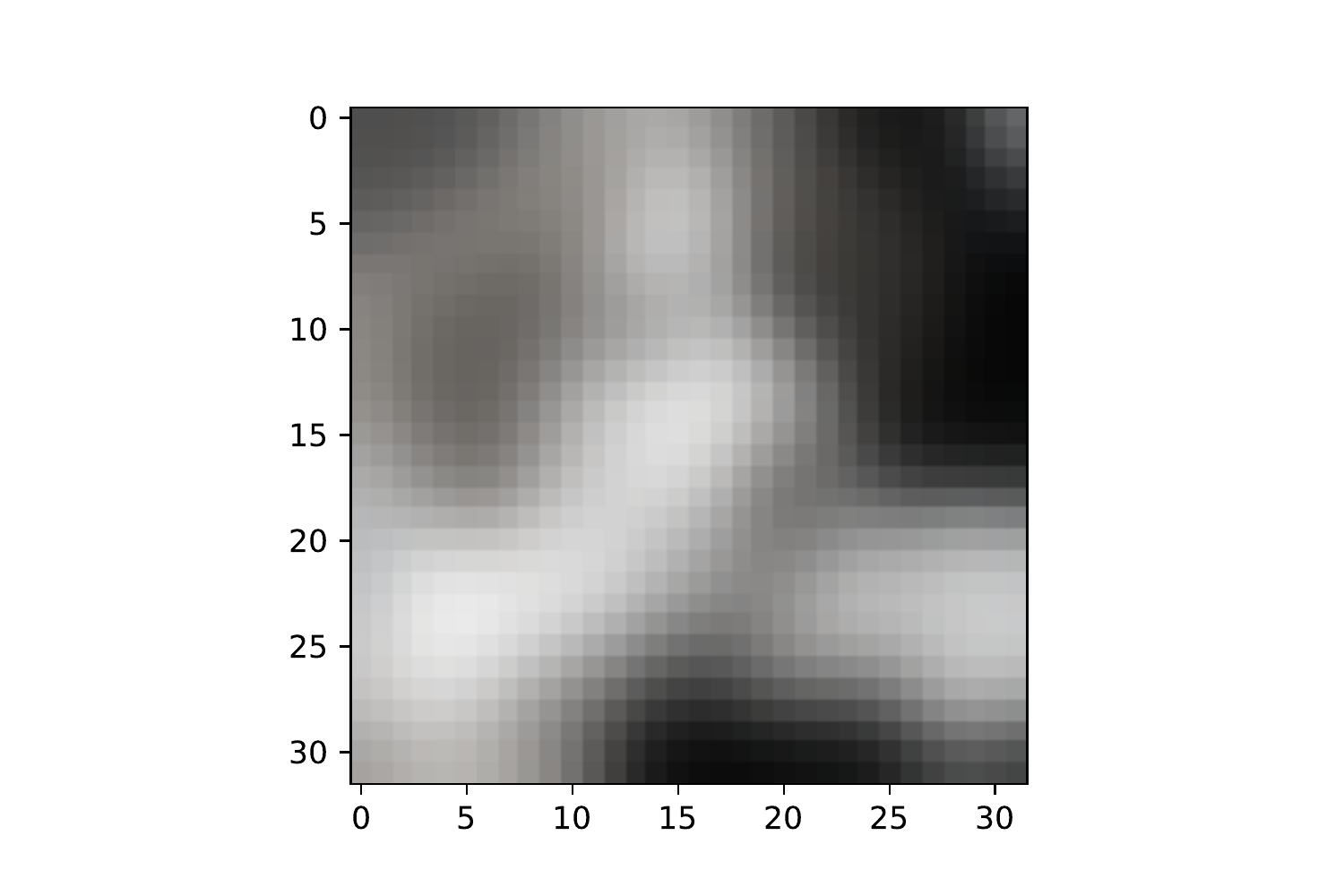} &
\hspace*{-10mm} \includegraphics[width=0.2\columnwidth]{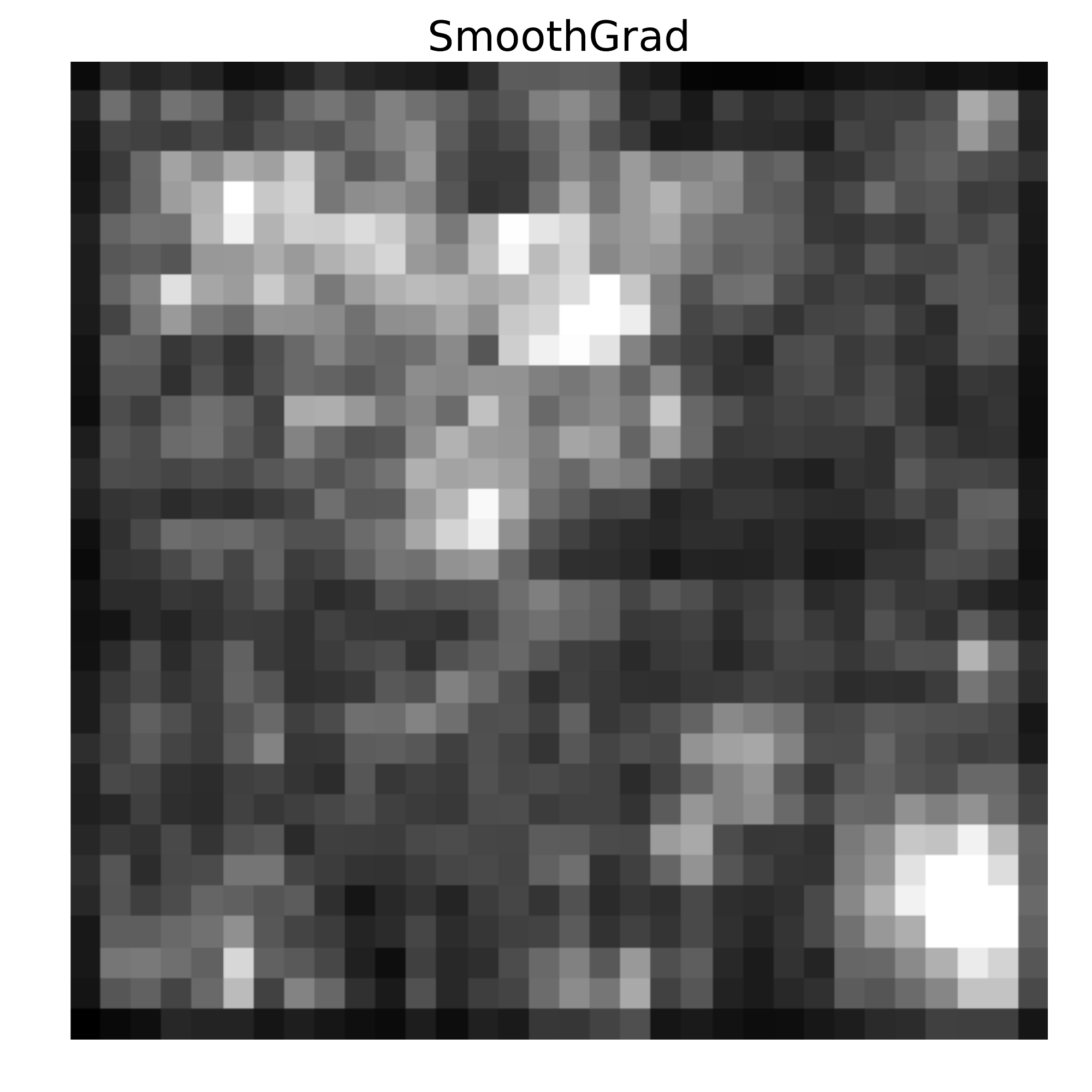} &
\hspace*{-5mm} \includegraphics[width=0.2\columnwidth]{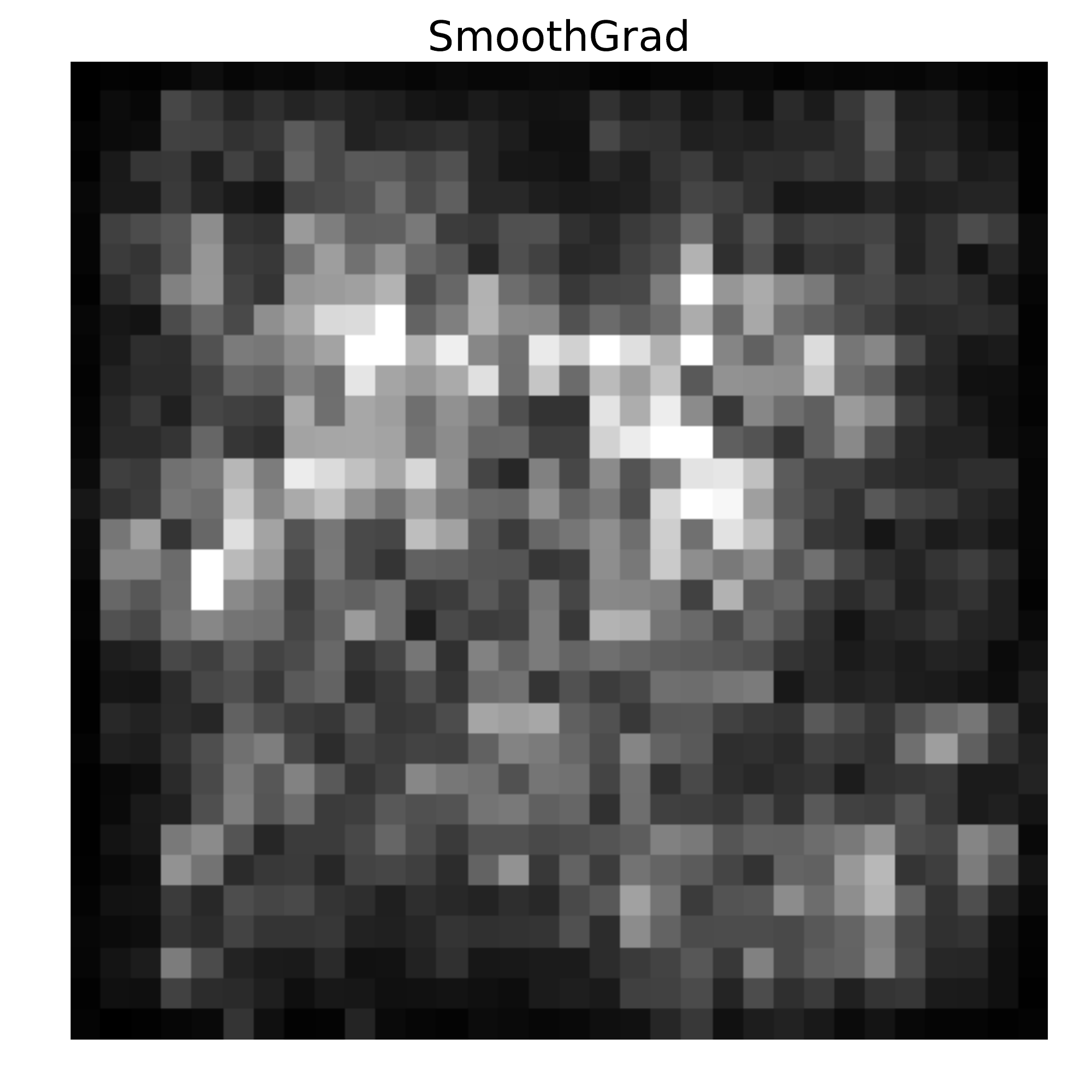} &
\hspace*{-5mm} \includegraphics[width=0.2\columnwidth]{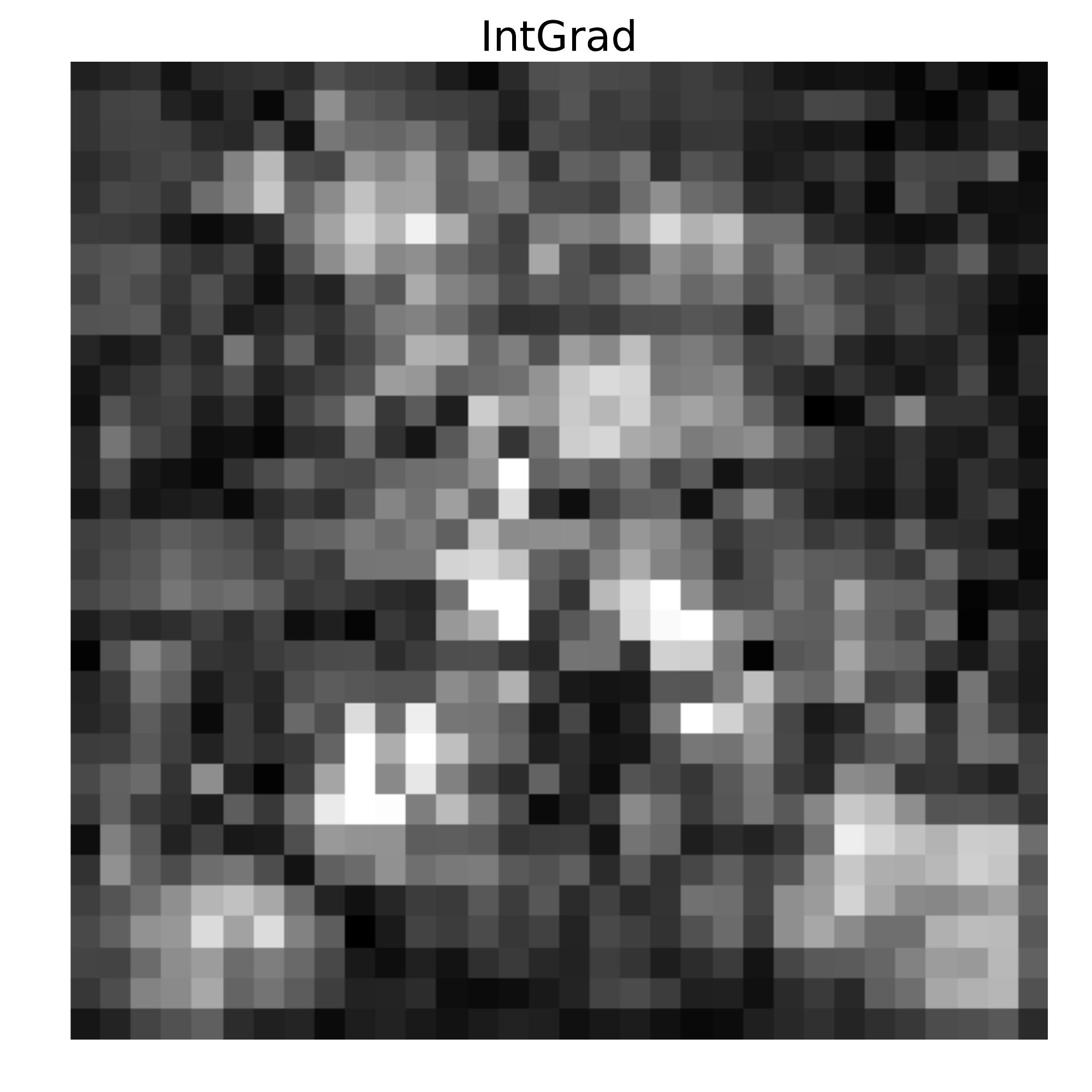} &
\hspace*{-5mm} \includegraphics[width=0.2\columnwidth]{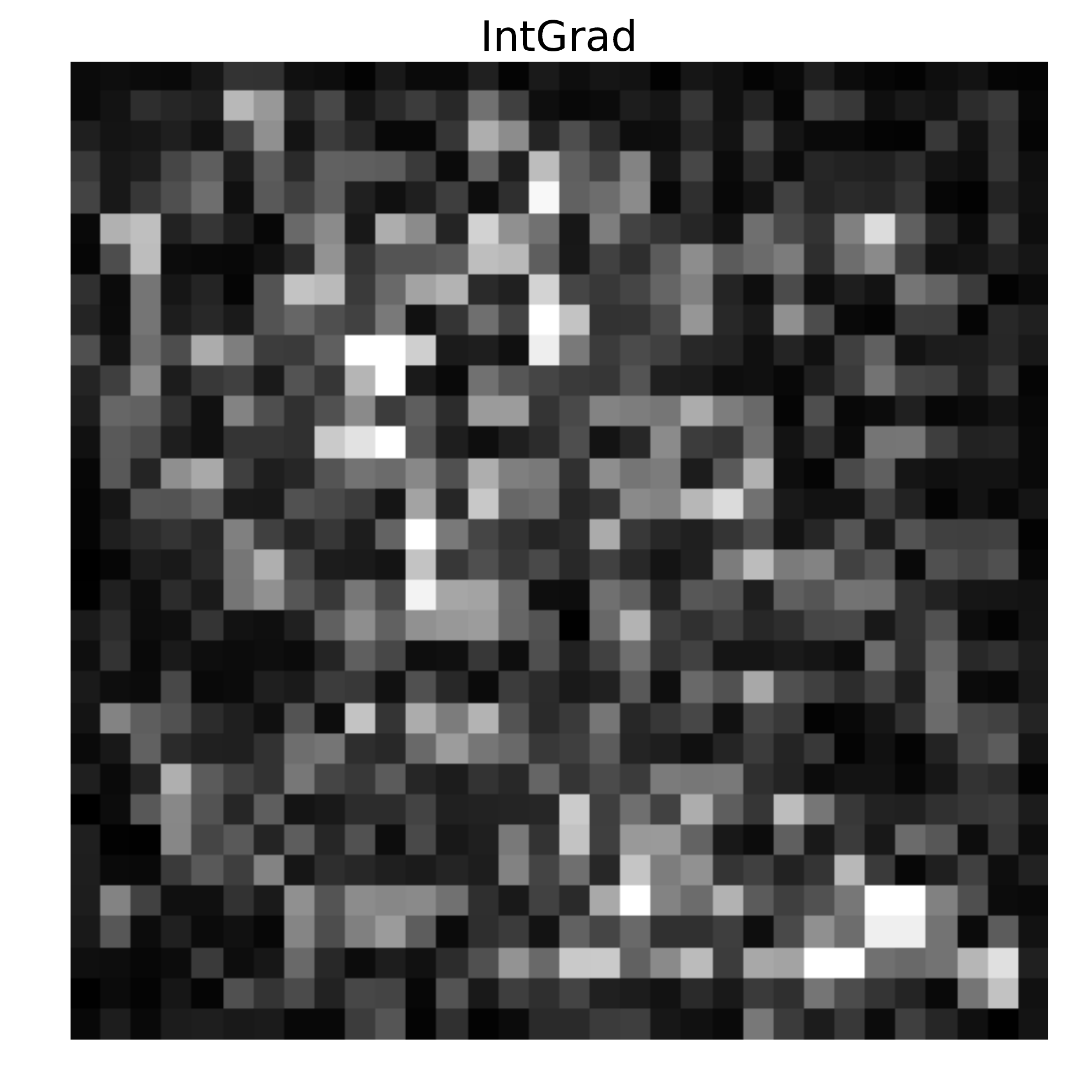} \\

\hspace*{-10mm} \includegraphics[width=0.35\columnwidth]{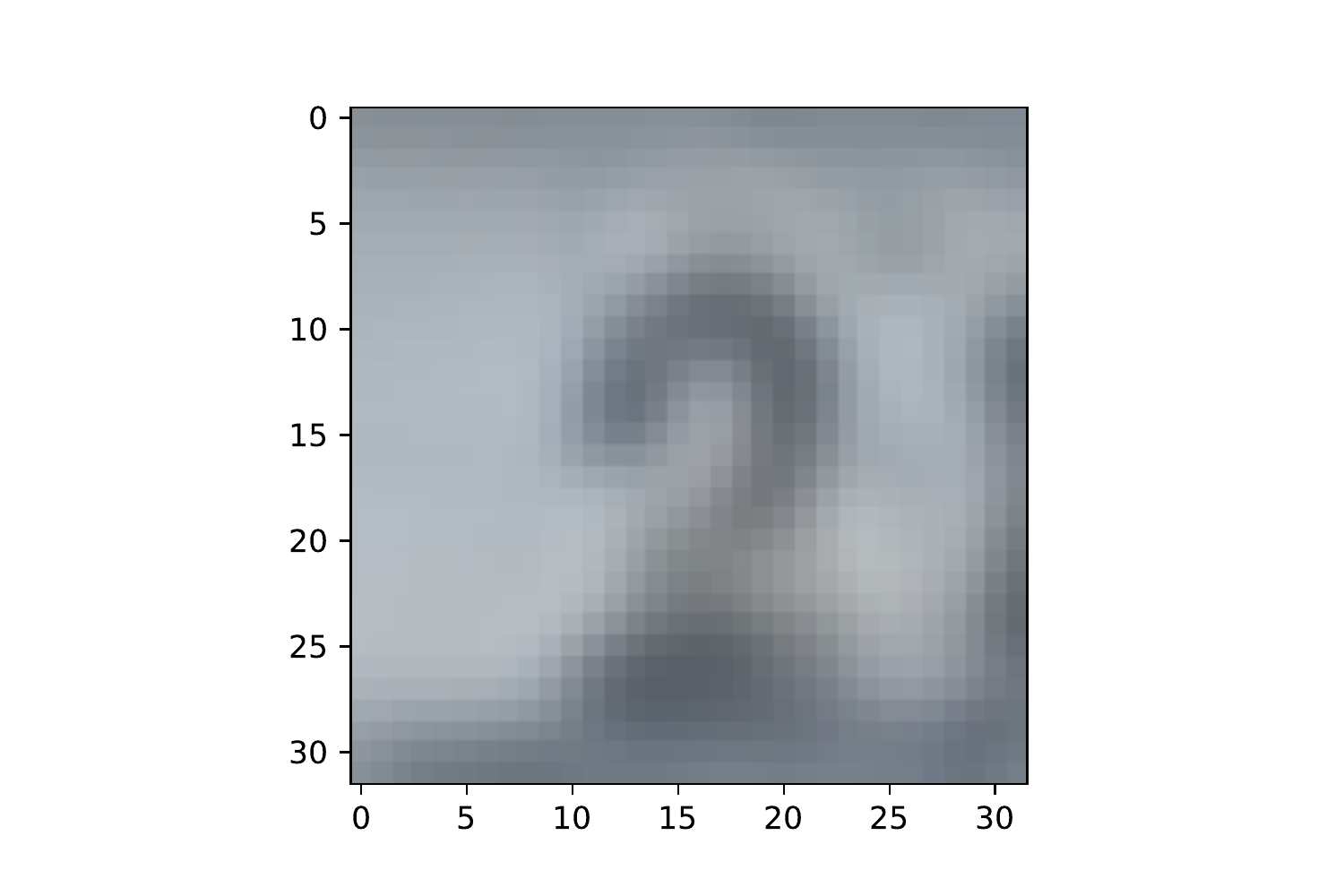} &
\hspace*{-10mm} \includegraphics[width=0.2\columnwidth]{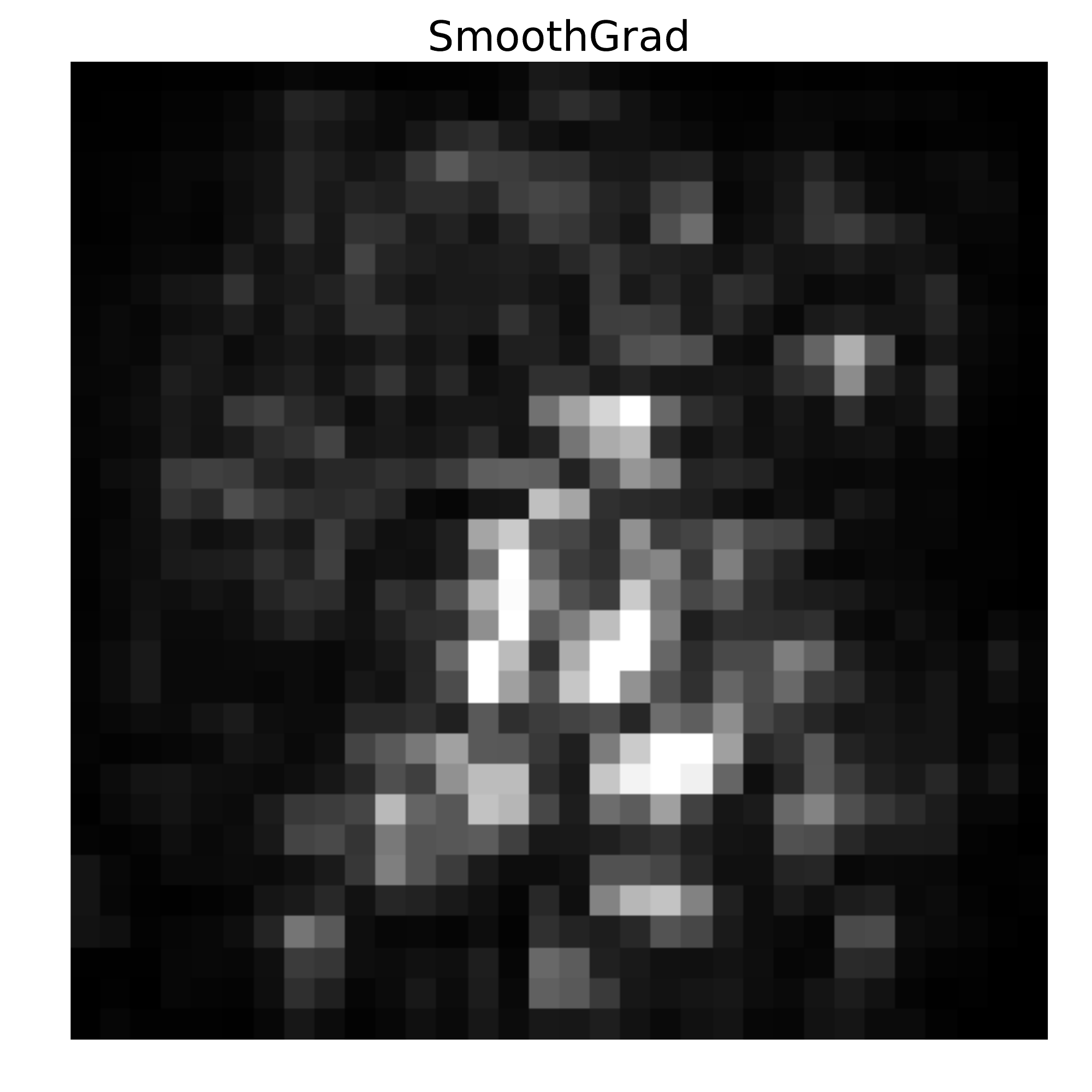} &
\hspace*{-5mm} \includegraphics[width=0.2\columnwidth]{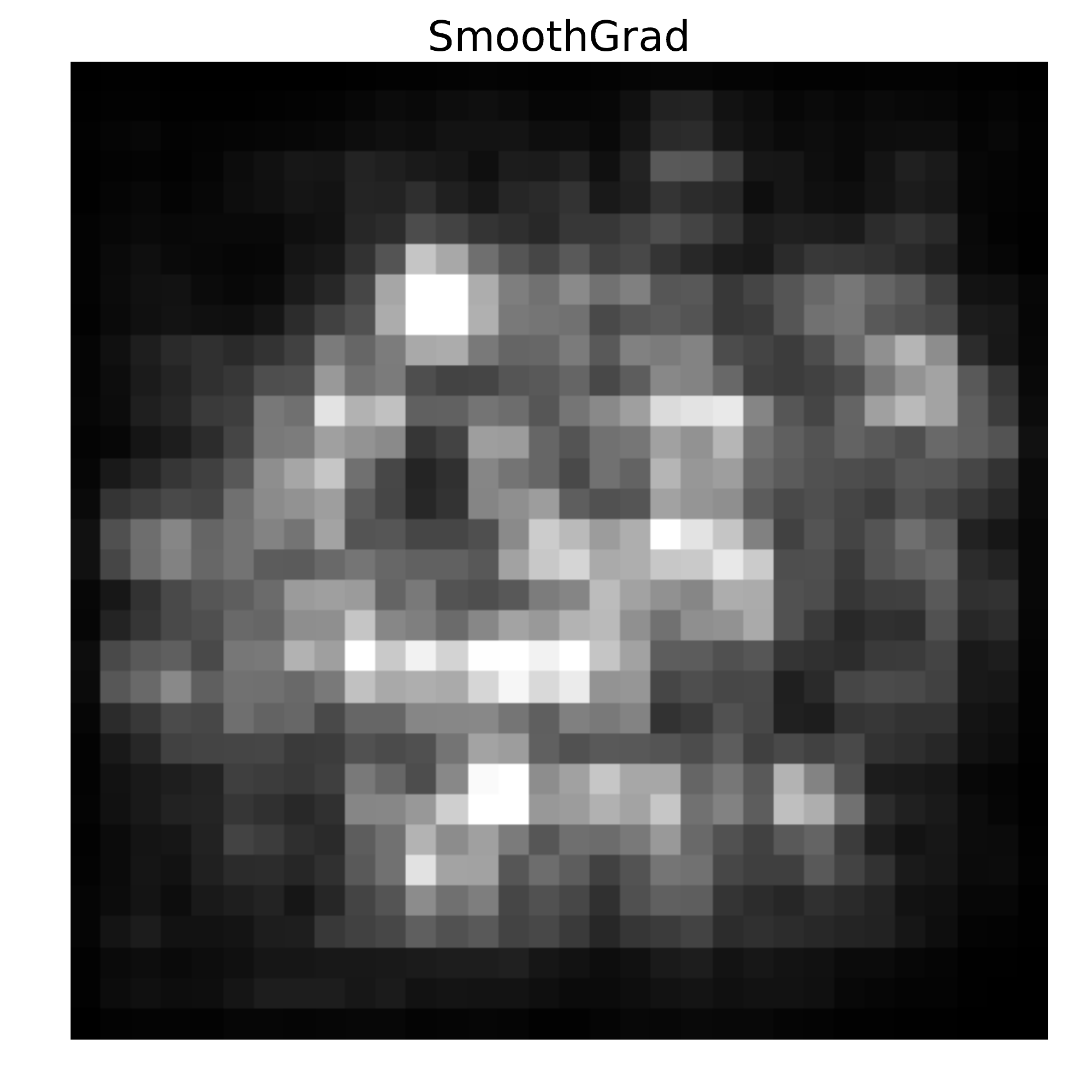} &
\hspace*{-5mm} \includegraphics[width=0.2\columnwidth]{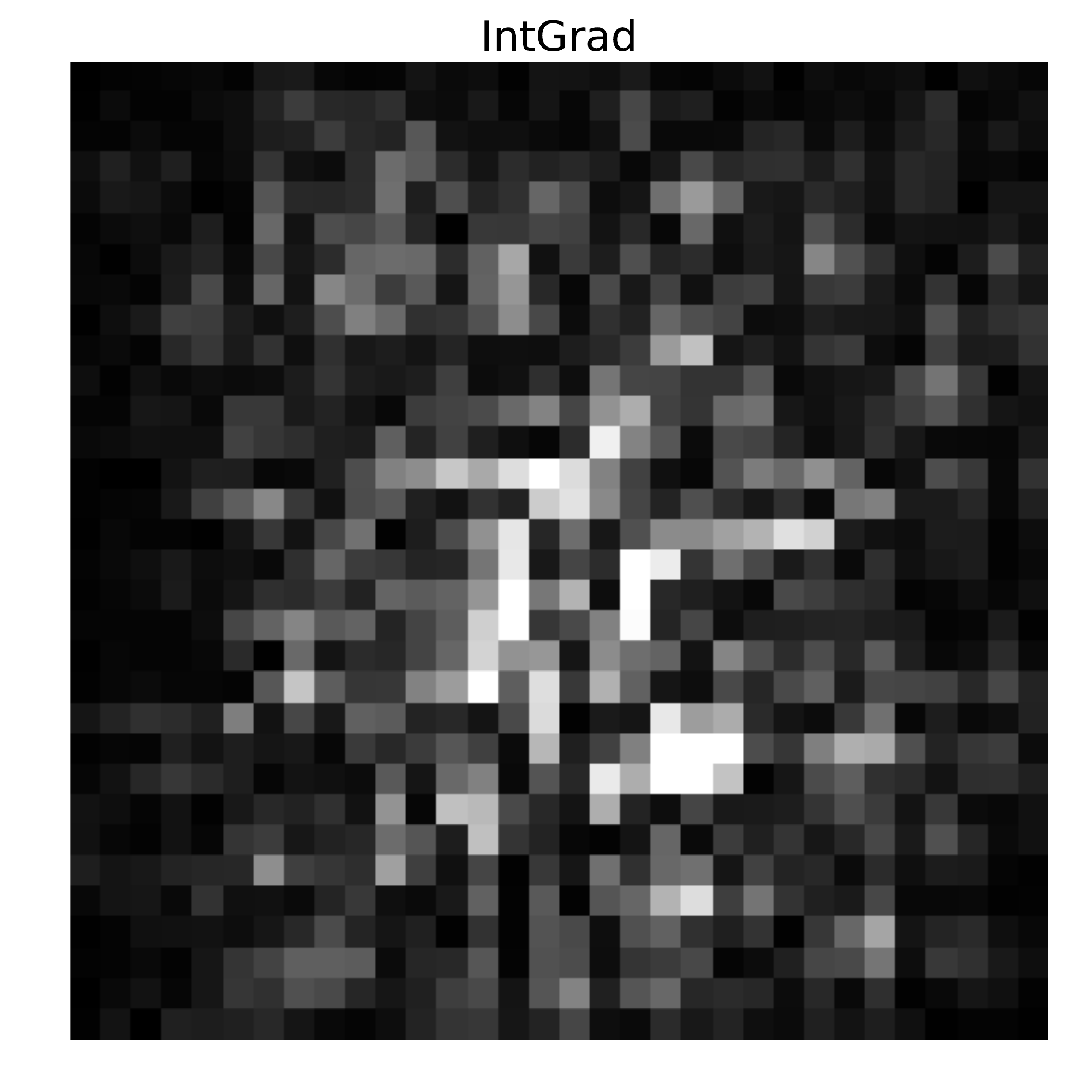} &
\hspace*{-5mm} \includegraphics[width=0.2\columnwidth]{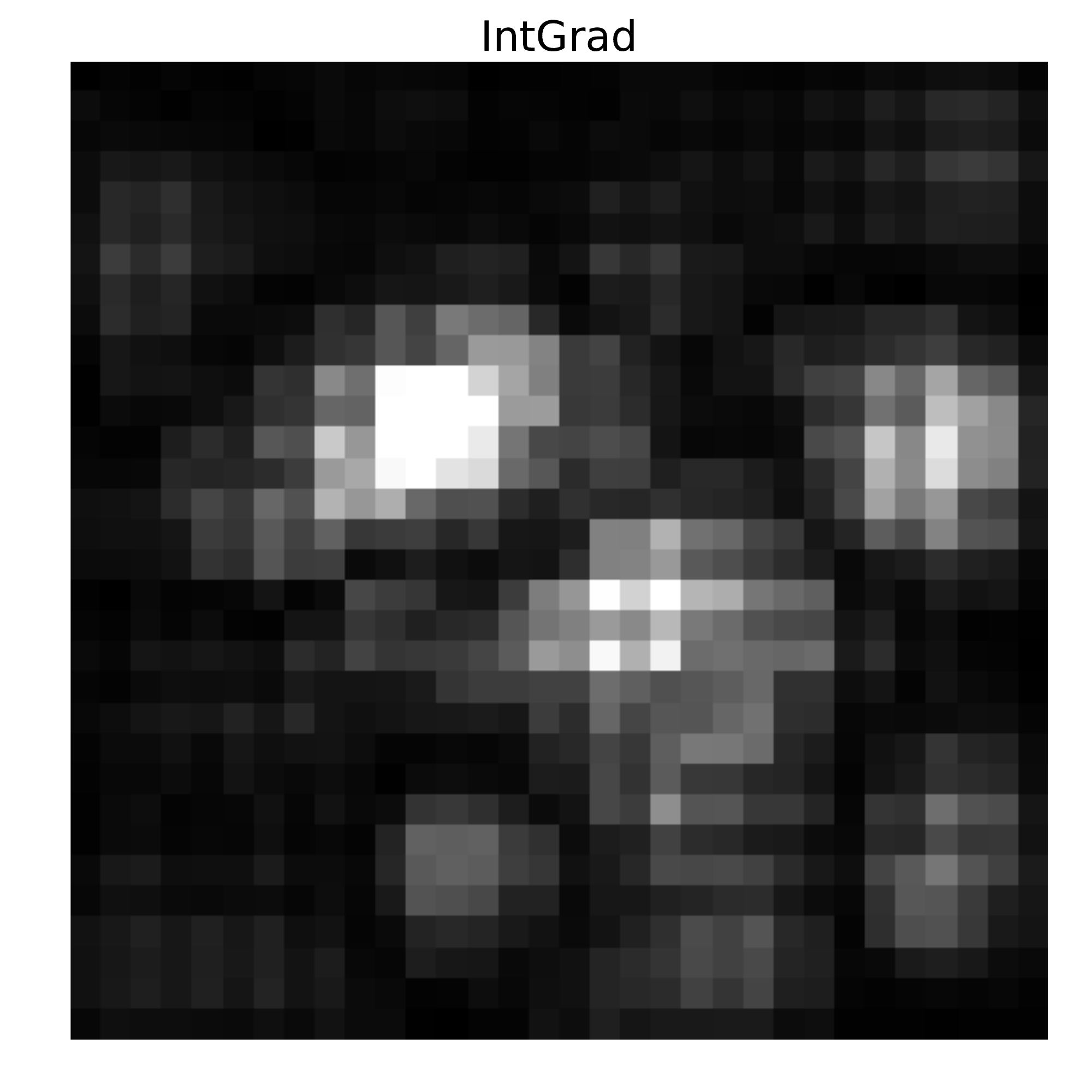} \\

\hspace*{-10mm} \includegraphics[width=0.35\columnwidth]{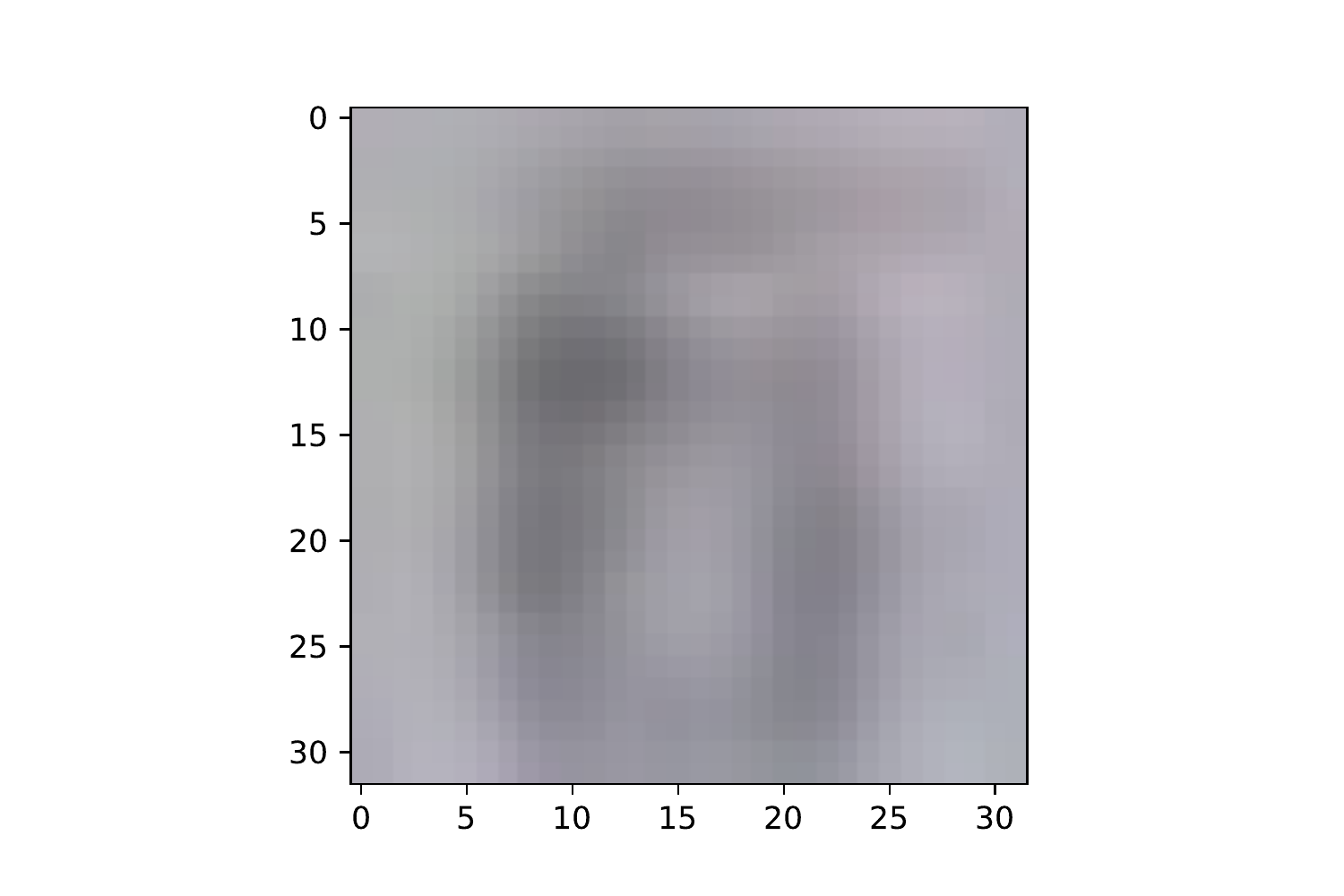} &
\hspace*{-10mm} \includegraphics[width=0.2\columnwidth]{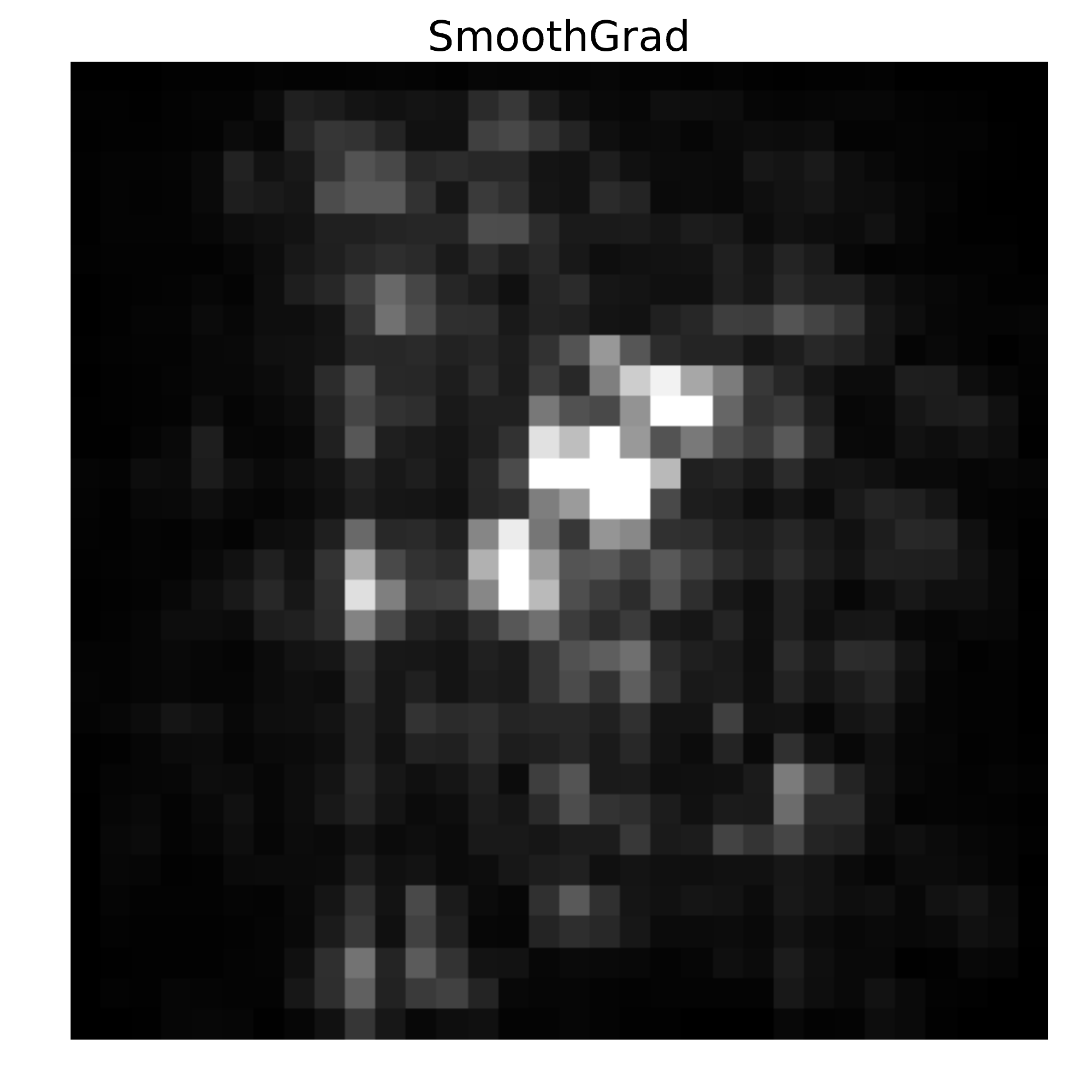} &
\hspace*{-5mm} \includegraphics[width=0.2\columnwidth]{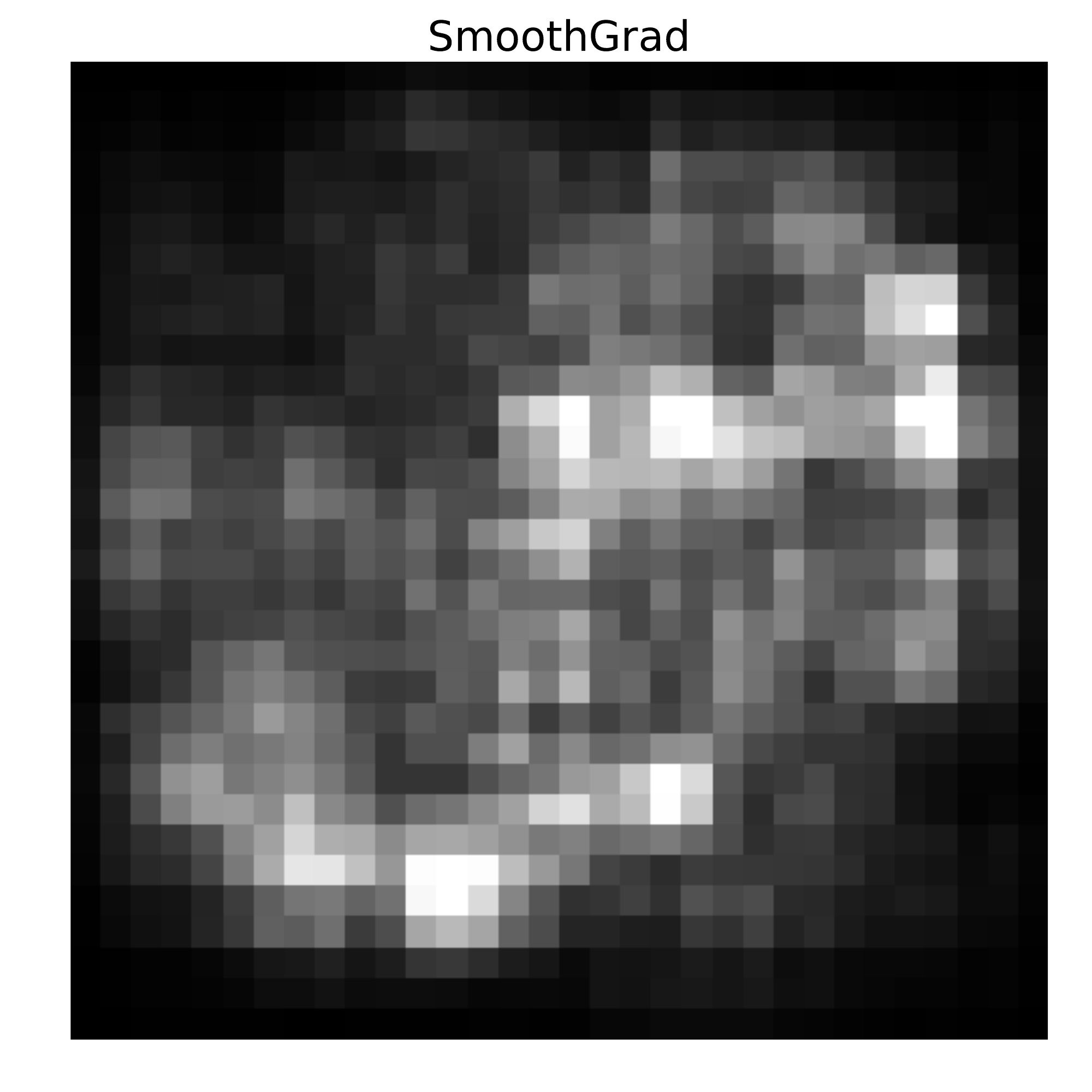} &
\hspace*{-5mm} \includegraphics[width=0.2\columnwidth]{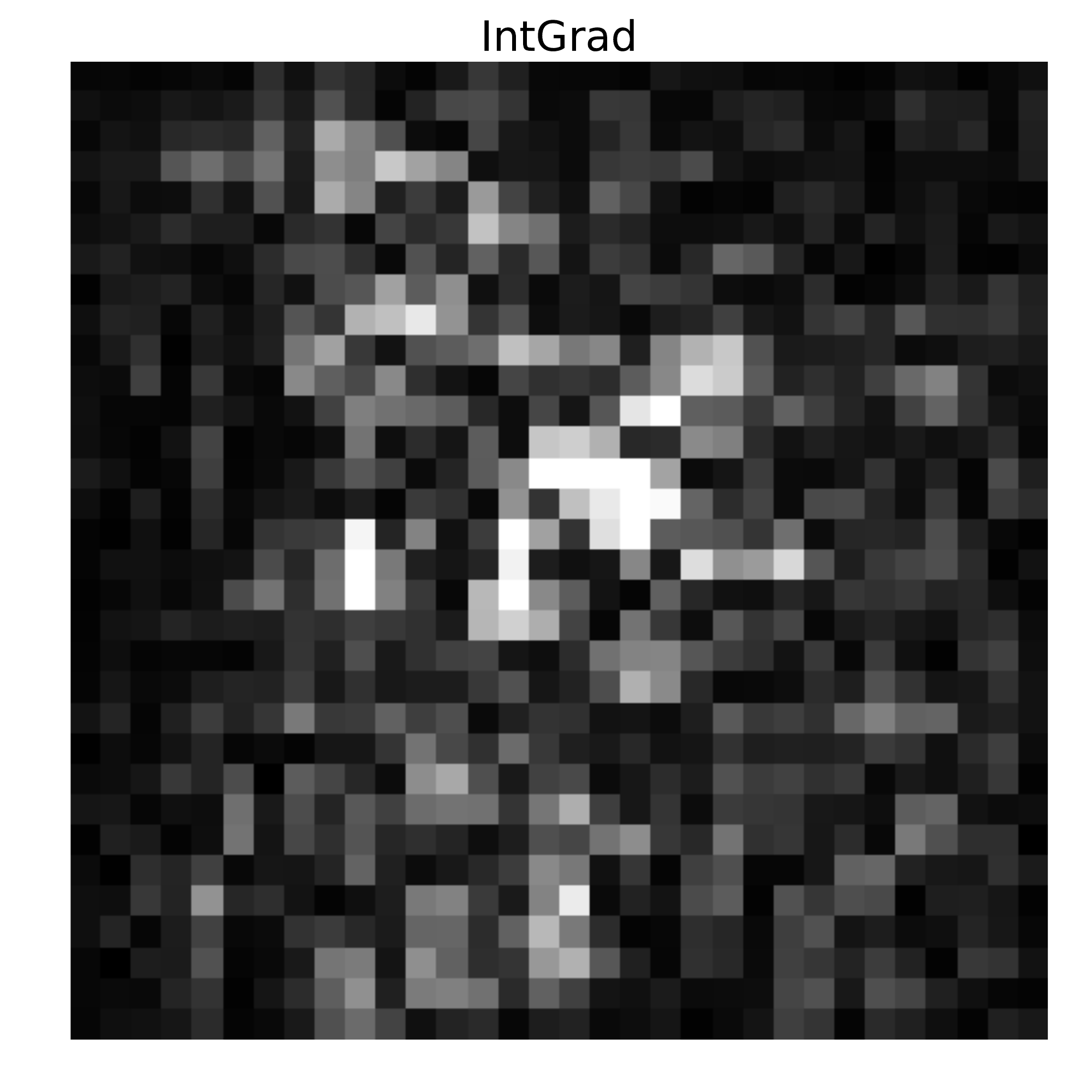} &
\hspace*{-5mm} \includegraphics[width=0.2\columnwidth]{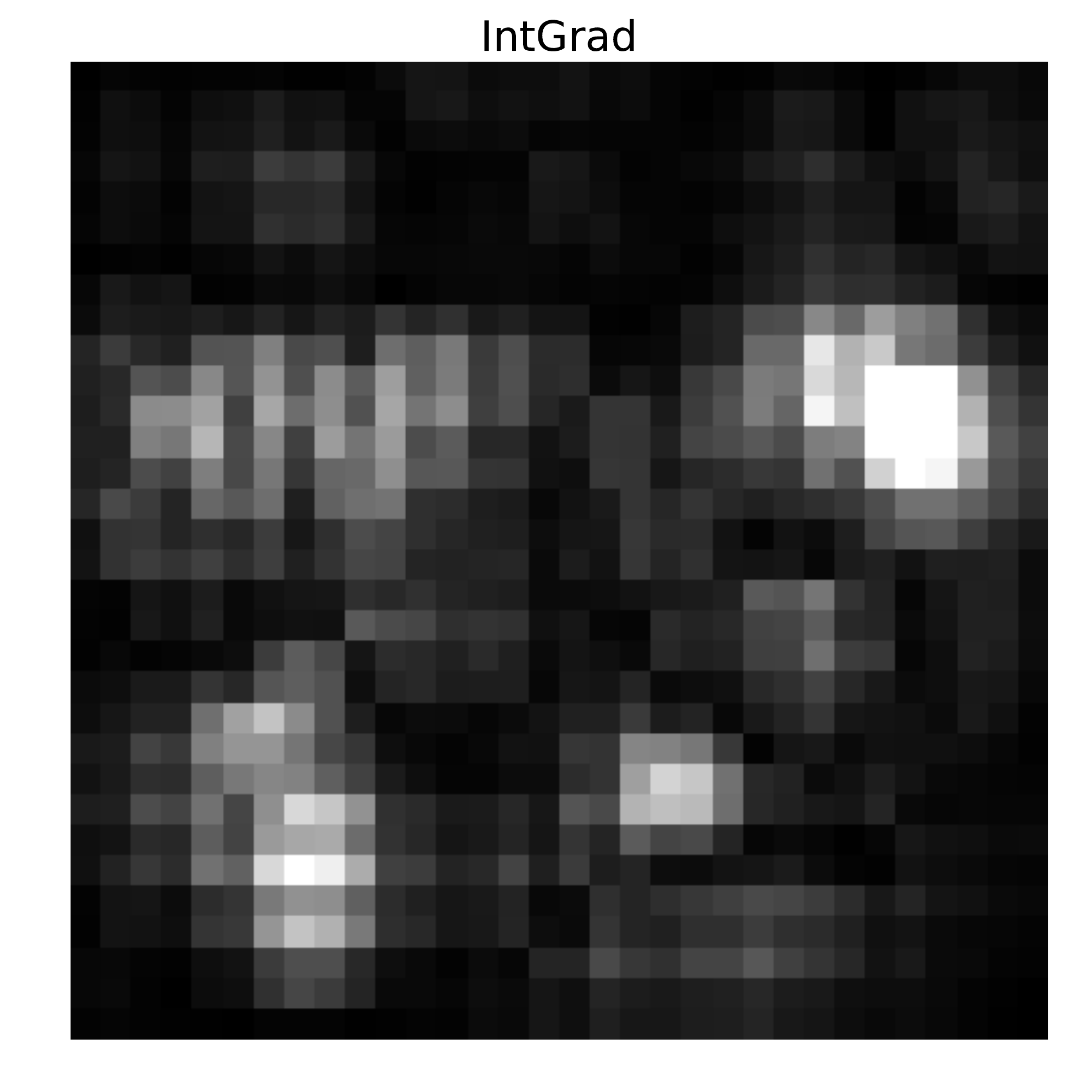} \\

\vspace{-3.5mm}
\end{tabular}
\caption{\small \textbf{Saliency maps for DUP and UVC models on the SVHN/CIFAR-10 disagreement task.} The plot shows two images from the blurred CIFAR-10 dataset and two images from the blurred SVHN dataset. The second column is SmoothGrad applied to the UVC model, and the third SmoothGrad applied to the DUP model. The third and fourth columns show IntGrad applied to the DUP and UVC models. We observe that the DUP and UVC models appear to be paying attention to different features of the dataset.}
\label{fig-saliency}
\end{figure}

\section{Details of DUP in the Medical Domain}
\label{sec-uncertainty-details}
As described in Section \ref{sec-medical-setting}, to train DUP models, we threshold the scores given by applying $U_{disagree}, U_{var}$ to the data $(x_i, \hat{\mathbf{p}}_i)$. Preliminary experiments in trying to directly regress onto the raw scores using mean-squared error performed poorly. 

We threshold the scores as follows. For $U_{var}$ we thresholded at approximately $2/9$, the variance when three doctors have more than an 'off by one' disagreement: more than a single disagreement, or a single grade disagreement.  

For $U_{disagree}$, where only the number of disagreements counts, we thresholded at $0.3$, to prioritize being sensitive enough to disagreement cases and having more than $20\%$ of the data marked as high disagreement. We also experimented with using soft targets for disagreement classification, but the results (Table \ref{table-label-count}) showed that this was less effective than than having the binary $0/1$ scores, likely because this makes the classification problem more like a regression. 

Our model consists of an Inception-v3 base, with the ImageNet head removed and a small (2 hidden layer, $300$ hidden units) fully connected neural network using Inception-v3 PreLogits to perform DUP. The full Inception-v3 network is trained with a batch size of $8$ and learning rate $0.001$ with the Adam optimizer. For training only the small neural network, we use the SGD with momentum optimizer, a batch size of $32$ and learning rate of $0.01$. 

\begin{table}
  \centering
  \hskip-1.0cm\begin{tabular}{lllllll}
    \toprule 
    \textbf{Model Type} & \textbf{$T_{test}$ AUC} & \textbf{Majority} & \textbf{Median} & \textbf{Majority$=1$} & \textbf{Median$=1$} & \textbf{Referable} \\
    \midrule
    %   [NEW] Disagree Label Count & \hspace{3mm} $79.0\%$ & \hspace{3mm} $78.7\%$ & \hspace{3mm} $81.6\%$ & \hspace{3mm} $79.0\%$ & \hspace{3mm} $84.7\%$  \\
    
    Disagree Soft Targets  & \hspace{3mm} $76.3\%$ & \hspace{3mm} $79.0\%$ & \hspace{3mm} $78.7\%$ & \hspace{3mm} $81.6\%$ & \hspace{3mm} $79.0\%$ &  \hspace{3mm} $84.7\%$  \\
      
      Disagree-P & \hspace{3mm} $\textbf{78.1\%}$ & \hspace{3mm} $\textbf{81.0\%}$ & \hspace{3mm} $80.8\%$ & \hspace{3mm} $\textbf{84.6\%}$ & \hspace{3mm} $\textbf{81.9\%}$ & \hspace{3mm} $\textbf{86.2\%}$   \\
      
      Disagree-PC & \hspace{3mm} $\textbf{78.1\%}$ & \hspace{3mm} $80.9\%$ & \hspace{3mm} $\textbf{80.9\%}$ & \hspace{3mm} $84.5\%$ & \hspace{3mm} $81.8\%$ & \hspace{3mm} $\textbf{86.2\%}$  \\
      
    \bottomrule
  \end{tabular}
  \caption{\small Using soft targets for disagreement prediction does not help in performance (AUC). Holdout AUC column corresponds to Disagreement Prediction Performance in Table \ref{table-var-prediction}, other columns refer to Table \ref{table-adj-agreement} in main text.}
\label{table-label-count}
\end{table}

\begin{table*}[t]
  \centering
  \hskip-1.0cm\begin{tabular}{lll}
    \toprule %\hspace*{-5mm}
    % \multicolumn{2}{c}{Part}                   \\
    % \cmidrule(r){1-2}
    \textbf{Task}  & \textbf{Model Type} & \textbf{Performance (AUC)} \\
    \midrule
     Variance Prediction      & Variance-E2E-2H & \hspace{5mm} $72.7\%$      \\
            \midrule
            \midrule
    Variance Prediction      & Variance-LR & \hspace{5mm} $72.4\%$    \\
    Disagreement Prediction   & Disagree-LR & \hspace{5mm} $75.9\%$ \\
    \bottomrule
  \end{tabular}
  \vspace{5mm}
  \caption{Additional results from table \ref{table-var-prediction}.}
   \label{table-additional-results-dup}
\end{table*}

\textbf{Prelogits, Calibration and Regularization}
Our training data for DUP models, $T^{(var)}_{train}, T^{(disagree)}_{train}$, only consists of $x_i$ with more than one label, and is too small to effectively train an Inception sized model end to end. Therefore, we use the prelogit embeddings of $x_i$ from a pretrained DR classification model (Histogram-E2E), and training smaller models on top of these embeddings. We do this both for the baseline, getting the Histogram-PC model, as well as the DUP models, Variance-PRC and Disagree-PC.

The \textit{C} suffix of all of these models corresponds to calibration on the logits. Following the findings of \cite{guo2017calibration}, we apply temperature scaling on the logits: we set the predictions of the model to be $f(\mathbb{z}/T)$ where $f$ is the softmax function, applied pointwise, and $\mathbb{z}$ are the logits. We initialize $T$ to $1$, and then split e.g. $T_{train}^{(disagree)}$ into a $T_{train}^{'(disagree)}$ and a $T_{valid}^{'(disagree)}$, with $10\%$ of the data in the validation set. We train as normal on $T_{train}^{'(disagree)}$, with $T$ fixed at $1$, and then train on $T_{valid}^{'(disagree)}$, by \textit{only} varying the temperature $T$, and holding all other parameters fixed. 

The use of Prelogit embeddings and Calibration gives the strongest performing baseline UVC and DUPs: Histogram-PC, Variance-PRC and Disagree-PC. For the Variance DUP, an additional regularization term is added to the loss by having a separate regressing on the raw variance value. 

\textbf{Additional Model: Variance-E2E} 
We tried a variant of Variance-E2E, Variance-E2E-2H, which has one head for predicting variance and the other for classifying, to enable usage of all the data. We then evaluate the variance head on $T_{test}$, but in fact noticed a small drop in performance, Table \ref{table-additional-results-dup}.

\textbf{Do we need the Prelogit embeddings?}
We tried seeing if we could match performance by training on pretrained classifier logits instead of the prelogit embeddings. Despite controlling for parameter difference by experimenting with more hidden layers, we found we were unable to match performance from the prelogit layer, Table \ref{table-additional-results-dup}, compare to Table \ref{table-var-prediction}. This demonstrates that some information is lost between the prelogit and logit layers.

\section{Additional Results: Entropy, Finite Sample Behavior and Convergence Analysis}
We performed additional experiments to further understand the properties of DUP and UVC models. For these experiments, we compare a representative DUP model, \textit{Disagree-P}, to a representative UVC model, \textit{Histogram-PC}.

Theorem \ref{thm-dup-unbiased} states that DUP offers benefits over UVC for \textit{concave} target uncertainty functions. This is a natural property for measures of spread, simply stating that the measure of spread increases with averaging (probability distributions). In the main text, we concentrate on two such specific uncertainty functions, $U_{var}$ and $U_{disagree}$, which are particularly suited to the domain. However, other standard uncertainty functions, such as entropy, are also concave. We test the performance of DUP (\textit{Disagree-P}) and UVC (\textit{Histogram-PC}) with $U_{entropy}$ as the target function. 

\begin{table*}[t]
  \centering
  \hskip-1.0cm\begin{tabular}{llll}
    \toprule %\hspace*{-5mm}
    % \multicolumn{2}{c}{Part}                   \\
    % \cmidrule(r){1-2}
     \textbf{Task}  & & \textbf{Model Type} & \textbf{Performance (AUC)} \\
    \midrule
    Entropy Prediction  & UVC    & Histogram-PC & \hspace{5mm} $75.5\%$    \\
    Entropy Prediction  & DUP  & Disagree-P & \hspace{5mm} $\textbf{77.2}\%$ \\
    \bottomrule
  \end{tabular}
  \vspace{5mm}
  \caption{DUP and UVC models trained with entropy as a target function. Again, we see that the DUP model outperforms the UVC model.}
   \label{table-entropy}
\end{table*}

The results are shown in Table \ref{table-entropy}, where we again see that DUP outperforms UVC.

\begin{figure}
\begin{tabular}{cc}
  \centering
  \hspace*{0mm}
 \includegraphics[width=0.45\columnwidth]{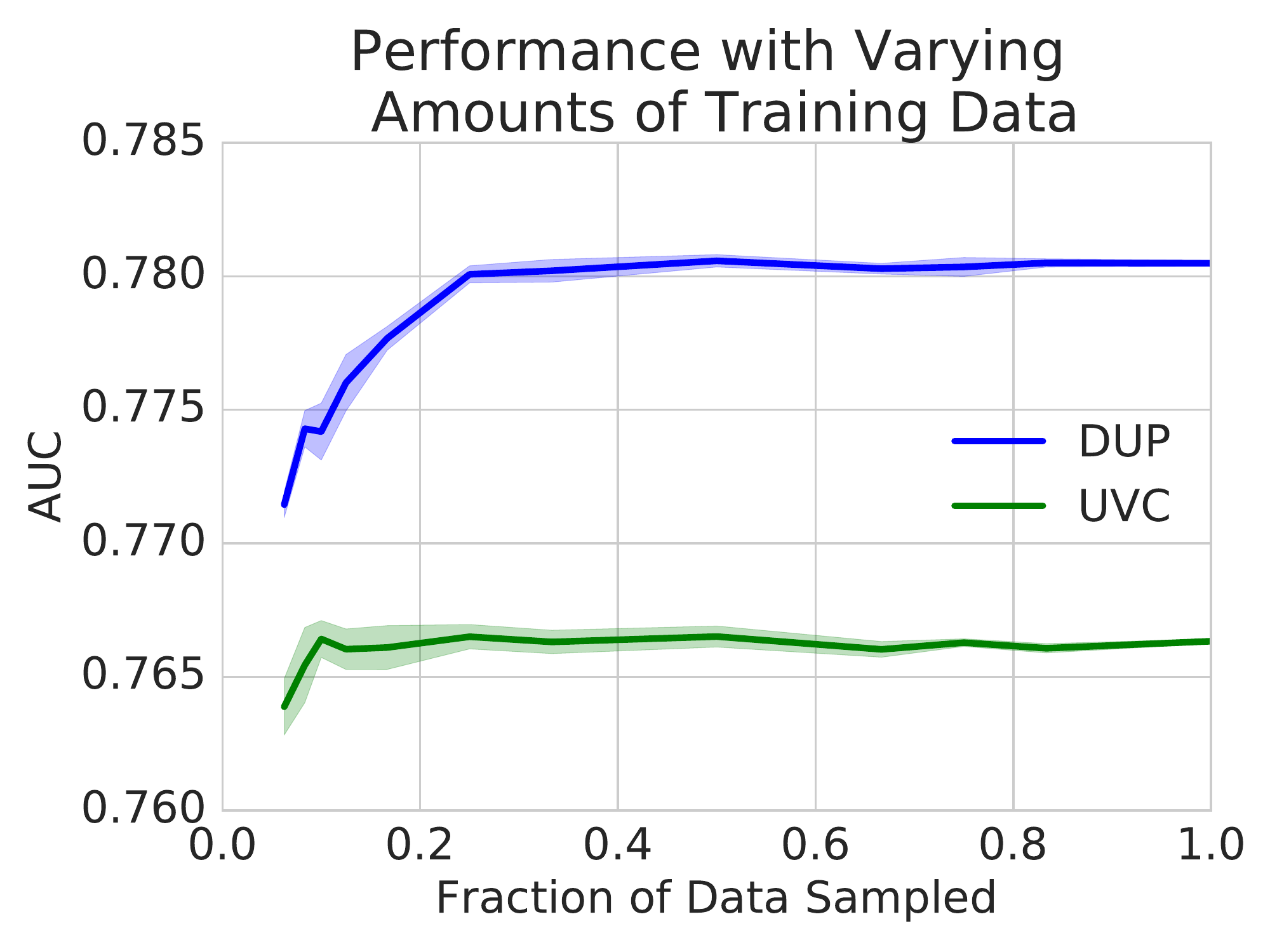} &
 \hspace*{2mm} \includegraphics[width=0.45\columnwidth]{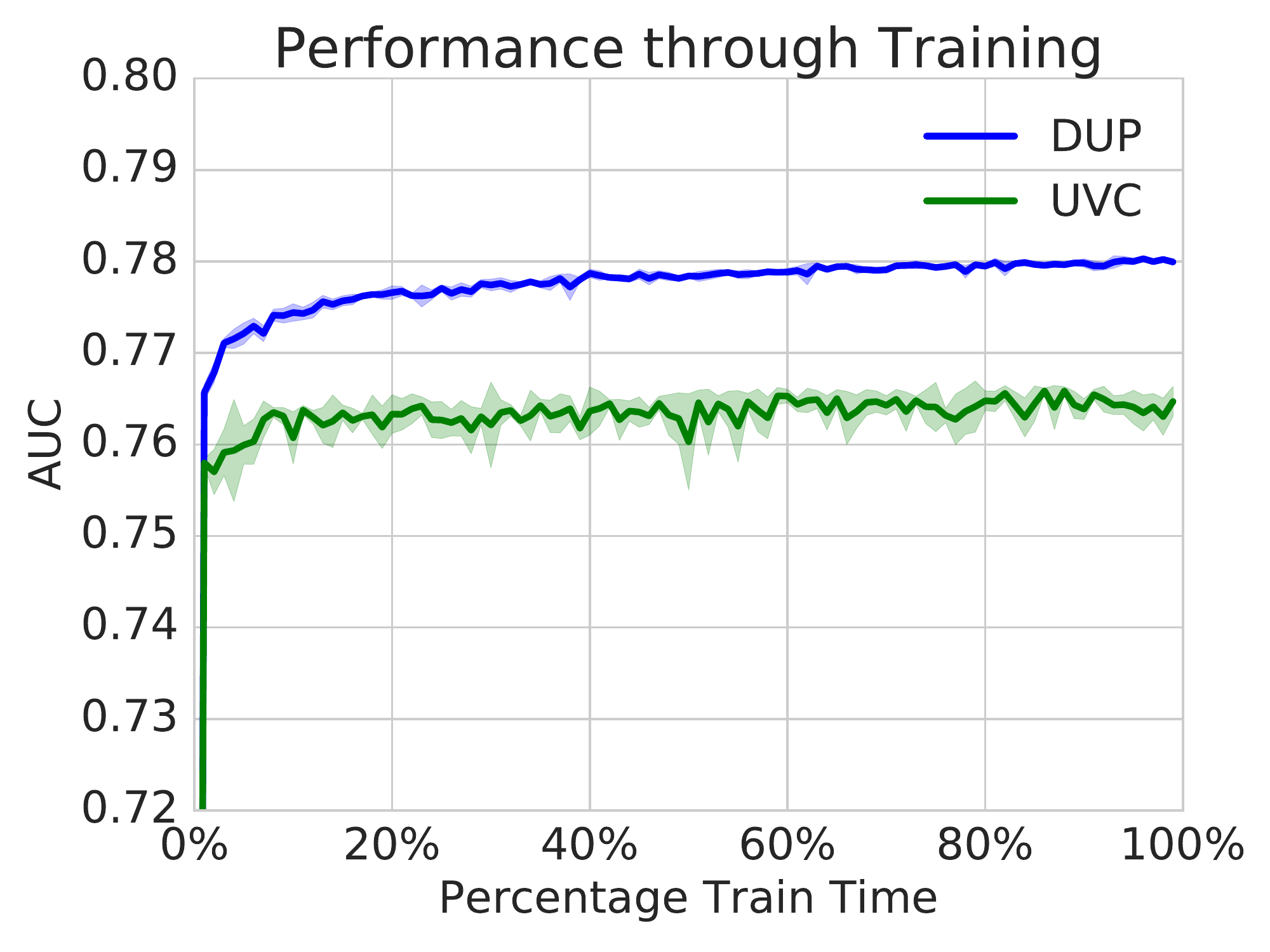}
 \end{tabular}
  \caption{\small \textbf{DUP and UVC performance during training and when varying train data size.} We study DUP (Disagree-P) and UVC (Histogram-PC) performance for varying amounts of training data. We find that the gap in performance is robust to variations in dataset size. For more than $30\%$ of the data, performance of DUP and UVC remains relatively constant, supporting the applicability of Theorem \ref{thm-dup-unbiased} in the finite data setting. The right plot looks at performance through training, with the gap appearing rapidly early in training, and slowly widening.}
  \label{fig-analysis-experiments}
  \vspace*{-5mm}
\end{figure}

We also study how model performance is impacted by different training set sizes (similar to the analysis in \cite{chen2018discriminatory}). We subsample different amounts of the original training set $T_{train}^{(disagree)}$, and train DUP and UVC models on this subset. The results over $5$ repeats of different subsamples and optimization runs are shown in Figure \ref{fig-analysis-experiments}.

We see that the performance gap between DUP and UVC is robust to train data size differences. Additionally, when $\geq 30\%$ of the training data is used, DUP and UVC performance remains relatively constant. This supports carrying over the results of Theorem \ref{thm-dup-unbiased}) and the full joint distribution $f(o, \mathbf{y})$ to the finite data setting.

We also study convergence of DUP and UVC models. We find that the performance gap between DUP and UVC manifests very early in training (Figure \ref{fig-analysis-experiments}, right plot), and continues to gradually widen through training.

\section{Background on the Wasserstein Distance}
\label{sec-app-wasserstein}
Given two probability distributions $f, g$, and letting $\Pi(f,g)$ be all product probability distributions with marginals $f$, $g$, the Wasserstein distance between $p, q$ is
\[ ||f - g||_w = \inf_{\pi \in \Pi(f,g)} \mathbb{E}_{(r,t) \sim \pi}\left[d(r, t)\right] \]
where $d(,)$ is some metric. This distance has connections to optimal
transport, and corresponds to the minimum cost (with respect to $d(,)$) of
moving the mass in distribution $f$ so that it is matches the
mass in distribution $g$. 
We can represent the amount of mass to move from $r$ to $t$
with $\pi(r, t)$. To be consistent with the mass at the start,
$f(r)$, and the mass at the end $g(t)$ we must have that $\int_{t'}
\pi(r, t') = f(r)$ and $\int_{r'} \pi(r', t) = g(t)$.

The result in the main text follows from the following theorem:
\begin{theorem}
If $f, g$ are (discrete) probability distributions and $g$ is a point mass distribution at $t_0$, then $\pi \in \Pi(f,g)$ is uniquely defined as:
\[ \pi(r, t) = 
\begin{cases}
                                   0 & \text{if $t \neq t_0$} \\
                                   $f(r)$ & \text{if $t=t_0$}
  \end{cases}
\]
\end{theorem}

\begin{proof}
The proof is direct: for $t \neq t_0$, we must have $\int_{r'} \pi(r', t) = g(t) = 0$, and so $\int_{t'} \pi(r, t') = \pi(r, t_0) = f(r)$. 
\end{proof}

We consider three different distances $d(,)$: 
\begin{enumerate}
\item[1] \textit{Absolute Value} $d(r,t) = |r - t|$. This follows an interpretation in which the grades are equally spaced, so that all successive grade differences have the same weight.
\item[2] \textit{2-Wasserstein Distance} $d(r,t) = (r - t)^2$, and, to make into a metric
\[ ||f - g ||_w =  \left( \mathbb{E}_{(r,t) \sim \pi}\left[d(r, t)\right] \right)^{1/2} \]
This adds a higher penalty for larger grade differences.
\item[3] \textit{Binary Disagreement} We set $d(r, t) = 0$ if $r = t$ and $1$ otherwise.
\end{enumerate}

\end{document}